\documentclass{article} 
\usepackage{graphicx}
\usepackage{subcaption}
\usepackage{booktabs} 
\usepackage{hyperref}

\usepackage[main, final]{neurips_2026}
\usepackage{hyperref}
\usepackage{placeins}

\usepackage{amsmath}
\usepackage{amssymb}
\usepackage{mathtools}
\usepackage{amsthm}
\usepackage[capitalize,noabbrev]{cleveref}
\usepackage{hhline}
\usepackage{makecell}
\usepackage{algorithm}
\usepackage{algpseudocode}
\usepackage{url}
\usepackage{makecell}
\usepackage{wrapfig}
\usepackage[capitalize,noabbrev]{cleveref}
\usepackage{enumitem}

\usepackage{pifont}
\usepackage{tabu}
\usepackage{multirow}
\usepackage{xcolor,colortbl}
\usepackage{hhline}
\usepackage{ragged2e}

\theoremstyle{definition}
\newtheorem{definition}{Definition}

\newtheorem{lemma}{Lemma}
\newtheorem{proposition}{Proposition}

\title{Meta-Scaling Laws: On the Invariance\\and Generality of Neural Scaling}

\author{%
  Xing Han \\
  Johns Hopkins University\\
  Baltimore, MD 21218, USA \\
  \texttt{xhan56@jhu.edu} \\
  \And
  Liu Ziyin \\
  NTT Research, MIT \\
  Cambridge, MA 02139, USA \\
  \texttt{ziyinl@mit.edu} \\
  \AND
  Suchi Saria \\
  Johns Hopkins University, Bayesian Health \\
  Baltimore, MD 21218, USA \\
  \texttt{ssaria1@jhu.edu} \\
  \And
  Paul Pu Liang \\
  MIT \\
  Cambridge, MA 02139, USA \\
  \texttt{ppliang@mit.edu} \\
}

\begin{document}
\raggedbottom 

\maketitle

\begin{abstract}
Neural scaling laws establish a predictable relationship between model performance and data or compute, offering crucial guidance for resource allocation in new domains and tasks. Yet such laws are most needed precisely where they are hardest to obtain: fitting one for a new model–task pair demands expensive sweeps that typically exhaust the very compute budget the law is meant to economize. We study \emph{meta-scaling laws}: relations that characterize when neural scaling behavior is preserved and predict how it changes under data transformations, enabling forecasts across domains when the shift can be approximated by a compatible transformation. We address this by identifying the right invariants: scaling laws are preserved under bijective (information-preserving) transformations of the data and modified in predictable, information-theoretically grounded ways under non-bijective transformations that lower its information resolution $\rho$: a single axis along which a law fit in one domain can be transported to another. We validate this across language, vision, and speech, and demonstrate two cross-domain applications: predicting scaling for language models trained on electronic health records from laws fit on general text, and predicting time-series classification scaling under varying levels of noise injection, recovering the data-scaling exponents to within $3\%$ error.
\end{abstract}

\section{Introduction} \label{sec:intro}
The development of foundation models has been governed by the empirical regularities known as Neural Scaling Laws (NSLs). These laws establish a predictable power-law relationship between a model's performance and the physical resources consumed during training, namely parameter count ($N$), dataset size ($D$), and compute budget. First formalized for language modeling \citep{kaplan2020scaling} and subsequently extended to vision \citep{zhai2022scaling}, multimodal learning \citep{aghajanyan2023scaling}, and beyond, these laws typically take the form $L \approx A N^{-\alpha} + B D^{-\beta}$, where $L$ is the pretraining loss and $A,B$ are coefficients, indicating that the pretraining loss decreases predictably as parameters or training data are added.

A key limitation of this framework is that fitting such a law is itself prohibitively expensive: for instance, deriving a scaling law for language models \citep{kaplan2020scaling} requires training a sweep of models spanning orders of magnitude in $N$ and $D$ and evaluating across multiple downstream tasks. This is infeasible in most real-world settings, where data, compute, or task diversity are all limited: yet these are precisely the settings in which practitioners most need to anticipate scaling behavior before committing to large-scale data collection or training runs \cite{zhang2025exploring, frey2023neural, chu2025revisiting, peng2025scaling}. A medical team, for example, may want to predict how a language model will scale on electronic health records without incurring the cost of training large models from scratch, a need that becomes acute when the target domain has limited data availability or when computational constraints preclude extensive experiments \cite{hernandez2021scaling, muennighoff2023scaling, goyal2024scaling, ye2024data}.
This motivates \emph{meta-scaling laws}: relations that characterize invariance and predict changes in neural scaling under data transformations. Such relations enable predictions from a source-domain scaling law when the target shift can be approximated by a compatible transformation, reducing the need for repeated scaling sweeps. While recent work has begun to explore how data quality \citep{subramanyam2025scaling}, volume \citep{goyal2024scaling}, and diversity \citep{chen2025revisiting} affect scaling, these studies remain confined to a single domain and task, and model data degradation solely as a reduction in effective sample size. As a consequence, they predict the same irreducible loss floor regardless of data quality, at odds with information-theoretic limits when task-relevant information is irreversibly destroyed.

In this work, we develop an information-theoretic framework for meta-scaling laws. We use data transformations to relate source and target domains through the task-relevant information they preserve, providing a basis for predicting how scaling behavior changes. We distinguish between \textbf{bijective transformations}, which are information-preserving and (we prove) leave the scaling exponents invariant, and \textbf{non-bijective transformations}, which lower the data's information content and degrade scaling through two mechanisms: \textit{variance inflation}, which raises the sample count needed for equivalent statistical precision, and \textit{optimal loss shift}, which raises the irreducible loss floor itself.
These two mechanisms yield a new scaling law parameterized by the information resolution $\rho(T) = I(X'; Y)/I(X;Y)$, the fraction of task-relevant mutual information that survives a transformation $T$ from clean input $X$ to its transformed counterpart $X'$, with $\rho=1$ recovering the standard Chinchilla form and $\rho<1$ capturing principled, information-theoretically grounded deviations from it. Our experiments demonstrate the accuracy of this law across language, vision, and speech domains under a range of bijective and non-bijective transformations, recovering the data-scaling exponents $\beta$ to within $3\%$ error when transferred across transformation types. We further show two cross-domain applications: predicting scaling for language models trained on electronic health records using a law fit on general text, and predicting how time-series foundation models scale under varying levels of real-world noise injection, in both cases obtaining accurate forecasts without training models at the target scale.

\textbf{Related Work.} Neural scaling laws were first formalized for language \citep{kaplan2020scaling, hoffmann2022training} and extended to vision \citep{zhai2022scaling, alabdulmohsin2023getting}, multimodal \citep{henighan2020scaling, cherti2023reproducible, aghajanyan2023scaling, shukor2025scaling}, and time-series \citep{edwards2024scaling} settings, generally treating data as a monolithic quantity. Recent theory connects sudden learning and power-law loss decay in training time through neural quadratic forms \citep{liu2026neural}. A complementary line predicts scaling behavior without exhaustive training, via observational capability spaces \citep{ruan2024observational}, intermediate checkpoints \citep{choshen2024hitchhiker}, loss-to-loss transfer \citep{brandfonbrener2024loss}, or domain-weighted mixtures \citep{shukor2025scaling}, but offers no principled account of when scaling exponents should change versus remain invariant. Data-dependent formulations move closer to this question, refining laws to account for noise \citep{bansal2022data}, quality-conditioned filtering \citep{goyal2024scaling, li2024scalingfilter}, diversity and syntheticity \citep{chang2024scaling}, and density-driven sub-scaling \citep{chen2025revisiting}, yet model degradation only as effective-sample-size reduction and predict a quality-independent loss floor. Finally, data transformations - bijective ones underlying normalizing flows and ICA \citep{rezende2015variational, papamakarios2021normalizing, hyvarinen2023nonlinear}, and non-bijective ones such as quantization and dimensionality reduction \citep{cosman2002using, lu2019learning, alqahtani2021deep, lang2024comprehensive} are pervasive in practice but have not been connected to scaling theory. Our framework unifies these threads through information resolution.

\section{Meta-Scaling Laws via Information Resolution} \label{sec:method}

In this section, we propose that NSLs are fundamentally governed by the \textbf{information resolution} $\rho(T) = I(X'; Y)/I(X;Y)$ of the dataset: the fraction of task-relevant mutual information that a transformation $T: X \mapsto X'$ preserves about the target $Y$. This single quantity links the two transformation classes we study: a transformation is bijective when $\rho(T) = 1$ (all task-relevant information is preserved, regardless of how the input is re-encoded), and non-bijective when $\rho(T) < 1$ (some task-relevant information is irreversibly destroyed). In Section \ref{subsec:invertible}, we show that NSLs are invariant under bijective transformations, since $\rho=1$ leaves the law's data-dependent term unchanged. In Section \ref{subsec:nonbijective}, we characterize how $\rho<1$ modifies NSLs through two distinct mechanisms — variance inflation and Bayes-risk elevation. To support our framework, in Appendix \ref{subsec:info_resolution}, we discuss how $\rho$ can be estimated in practice and used to predict scaling behavior across domains.

\subsection{NSLs are Robust to Bijective Transformations} \label{subsec:invertible}

Real-world model development routinely subjects data to transformations: tokenizers re-encode text, frequency-domain conversions re-represent audio, color-space changes re-parameterize images, and normalizing flows re-distribute features. Each one alters the surface form of the input while, in many cases, preserving every bit of information relevant to the downstream task. We formalize this class first, since it provides the natural baseline against which the more interesting non-bijective case is compared: any change in scaling behavior under a bijective transformation must come from optimization dynamics, not from a change in what the data fundamentally encodes.

\begin{definition}[Bijective Transformation]
\label{def:bijective}
A transformation $T: \mathcal{X} \to \mathcal{X}'$ is \emph{bijective} if it satisfies:
\begin{enumerate}
    \item \textbf{Injectivity}: For all $x_1, x_2 \in \mathcal{X}$, $T(x_1) = T(x_2)$ implies $x_1 = x_2$.
    \item \textbf{Surjectivity}: For every $x' \in \mathcal{X}'$, there exists $x \in \mathcal{X}$ such that $T(x) = x'$.
\end{enumerate}
Equivalently, $T$ is bijective if and only if there exists an inverse mapping $T^{-1}: \mathcal{X}' \to \mathcal{X}$ such that $T^{-1}(T(x)) = x$ for all $x \in \mathcal{X}$ and $T(T^{-1}(x')) = x'$ for all $x' \in \mathcal{X}'$.
\end{definition}
We consider two classes of bijective transformations.

\textbf{Linear Transformations.} Given a dataset with representations in $\mathbb{R}^d$, a linear transformation is defined as:
$T_{\text{linear}}: \mathbb{R}^d \to \mathbb{R}^d, ~ x \mapsto Wx ~$
where $W \in \mathbb{R}^{d \times d}$ is a full-rank matrix with $\det(W) \neq 0$. The bijectivity follows from the existence of the inverse transformation $T^{-1}_{\text{linear}}(x') = W^{-1}x'$.

\textbf{Nonlinear Bijective Transformations.} A nonlinear transformation $T_{\text{nonlinear}}: \mathbb{R}^d \to \mathbb{R}^d$ is bijective if it is a diffeomorphism: a smooth bijection with a smooth inverse. Such transformations can be constructed through compositions of bijective layers:
$
T_{\text{nonlinear}} = f_L \circ f_{L-1} \circ \cdots \circ f_1
$
where each $f_\ell$ is an bijective function. The composition of bijective functions is bijective, since $(f_L \circ \cdots \circ f_1)^{-1} = f_1^{-1} \circ \cdots \circ f_L^{-1}$.
A transformation $T$ is information-preserving with respect to a random variable $X$ if the entropy is conserved: $H(T(X)) = H(X)$. Both discrete and continuous random variables under diffeomorphisms satisfy this property \citep{nair2006entropy, parra1996statistical}. We now provide theoretical grounding for why NSLs remain invariant under bijective transformations.

\begin{lemma}[Mutual Information Invariance]
\label{thm:mi_invariance}
Let $X$ be the input data and $Y$ be the target labels. For any bijective transformation $T$, the mutual information is preserved:
$I(T(X); Y) = I(X; Y).$
\end{lemma}

\begin{proof}
By the data processing inequality, for any $T$, we have $I(T(X); Y) \leq I(X; Y)$. Since $T$ is bijective, we can also apply the transformation $T^{-1}$ to obtain $I(X; Y) = I(T^{-1}(T(X)); Y) \leq I(T(X); Y)$. Combining both inequalities yields $I(T(X); Y) = I(X; Y)$.
\end{proof}
By mutual information definition: $I(X; Y) = H(Y) - H(Y|X)$, since $H(Y)$ depends only on the target distribution and $I(T(X); Y) = I(X; Y)$ by Lemma~\ref{thm:mi_invariance}, we also have $H(Y|T(X)) = H(Y|X)$.

\textbf{Connection to Model Performance.} The Chinchilla Scaling Law \citep{hoffmann2022training} specifies that
\begin{equation}
    L(N, D) = \frac{A}{N^{\alpha}} + \frac{B}{D^{\beta}} + E,
    \label{eq:Chinchilla_law}
\end{equation}
where $N$ is the number of training parameters, $D$ is the number of training tokens, $\alpha/\beta$ is the exponent for model/data scaling, and $E$ is the irreducible loss.
Lemma~\ref{thm:mi_invariance} has direct implications for several learning-theoretic quantities below that govern NSL behavior. 

\begin{proposition}[Bayes Risk Invariance]
\label{prop:bayes_risk}
Let $T: \mathcal{X} \to \mathcal{X}'$ be a bijective transformation. The Bayes-optimal risk satisfies:
\begin{equation}
R^*_{T(X)} := \inf_{f: \mathcal{X}' \to \mathcal{Y}} \mathbb{E}[\ell(f(T(X)), Y)] = \inf_{g: \mathcal{X} \to \mathcal{Y}} \mathbb{E}[\ell(g(X), Y)] =: R^*_X
\end{equation}
That is, the irreducible loss $E$ in the scaling law remains unchanged under bijective transformations.
\end{proposition}

\begin{proposition}[Sample Complexity Preservation]
\label{prop:sample_complexity}
The sample complexity for achieving $\epsilon$-excess risk is preserved:
$
n_\epsilon(T(X), Y) = n_\epsilon(X, Y),
$
where $n_\epsilon$ denotes the minimum number of samples required to achieve excess risk at most $\epsilon$ with high probability. Consequently, the data scaling exponent $\beta$ remains invariant.
\end{proposition}

\begin{proposition}[Model Capacity Preservation]
\label{prop:model_capacity}
The model capacity to approximate the optimal predictor is preserved, leaving the parameter scaling exponent $\alpha$ invariant.
\end{proposition}

\begin{lemma}[NSL Invariance under Bijective Transformations]
\label{lemma:scaling_invariance}
Let $\mathcal{D} = \{(x_i, y_i)\}_{i=1}^n$ be a dataset drawn from distribution $P_{XY}$. Define the transformed dataset $\mathcal{D}' = \{(T(x_i), y_i)\}_{i=1}^n$. If the test loss on $\mathcal{D}$ follows the scaling law in Eq.\eqref{eq:Chinchilla_law}, then the test loss on $\mathcal{D}'$ follows: 
$
L'(N, D) = \frac{A'}{N^{\alpha}} + \frac{B'}{D^{\beta}} + E
$
with identical exponents $(\alpha, \beta)$ and irreducible loss $E$.
\end{lemma}

The core intuition is that bijectivity allows lossless translation between predictor spaces: any predictor $f$ on $X$ maps to $g=f\circ T^{-1}$ on $T(X)$ with identical risk, preserving Bayes risk ($E$), sample complexity ($\beta$), and approximation rates ($\alpha$). Lemma \ref{lemma:scaling_invariance} combines these invariances directly.
Proofs of Proposition \ref{prop:bayes_risk} - \ref{prop:model_capacity} can be found in Appendix \ref{app:invariant_proofs}. 

\subsection{NSLs are Dependent on Information Resolution} \label{subsec:nonbijective}

\begin{table*}[t]
\centering
\caption{Comparison of our framework with two recent scaling law formulations under data degradation: Quality-Aware scaling laws \citep{subramanyam2025scaling} and Sub-Optimal \citep{chen2025revisiting} scaling laws.}
\label{tab:comparison}
\setlength{\tabcolsep}{1.5pt}
\scalebox{0.79}{
\begin{tabular}{lccc}
\toprule
\textbf{Aspect} & \textbf{Quality-Aware} & \textbf{Sub-Optimal} & \textbf{Ours (Info-Resolution)} \\
\midrule
Formulation & $\frac{A}{N^\alpha} + \frac{B}{D^\beta Q^\gamma} + E$ & $\frac{\lambda_N R_N}{N^{\alpha_N}} + \frac{\lambda_D R_D}{D^{\alpha_D}} + E$ & $\frac{A}{N^\alpha} + \frac{B}{D^\beta} \cdot \rho(T)^{-\nu} + E + \kappa(1 - \rho(T))^\mu$ \\[6pt]
Mechanisms & Single (effective data) & Single (decay factors) & Dual (inflation + shift) \\[3pt]
Irreducible loss $E$ & Constant & Constant & $\rho$-dependent: $E + \kappa(1 - \rho)^\mu$ \\[3pt]
Degradation parameter & Data quality $Q$ & Data density $\rho_{\text{density}}$ & Transformation property $\rho(T)$ \\[3pt]
Interpretation & Fewer useful samples & Redundancy $\to$ diminishing returns & Variance inflation + optimal loss shift \\[3pt]
Asymptotic ($D \to \infty$) & $L \to E$ & $L \to E$ & $L \to E + \kappa(1 - \rho(T))^\mu$ \\
\bottomrule
\end{tabular}
}
\end{table*}

In practice, many of the transformations applied to training data are not bijective. Quantization collapses continuous values onto a finite grid \citep{lu2019learning}; low-rank projection discards directions in feature space \citep{markovsky2012low}; aggressive deduplication, summarization, or filtering removes content entirely \citep{gionis2007clustering}; and stochastic corruption \citep{subramanyam2025scaling} injects ambiguity into the input-output relationship. Each of these operations can irreversibly destroy task-relevant information about $Y$, driving $\rho(T)<1$ and altering the scaling law in ways that bijective re-encodings cannot. We now define this class formally and characterize how it modifies scaling behavior.

\begin{definition}[Non-Bijective Transformation]
\label{def:non_bijective}
A transformation $T: \mathcal{X} \to \mathcal{X}'$ is \emph{non-bijective} if and only if at least one of the following holds:
\begin{enumerate}
    \item \textbf{Non-injectivity}: There exist distinct $x_1, x_2 \in \mathcal{X}$ such that $T(x_1) = T(x_2)$.
    \item \textbf{Non-surjectivity}: There exists $x' \in \mathcal{X}'$ such that $x' \notin \text{range}(T)$.
\end{enumerate}
\end{definition}

We define the \textit{Information Resolution} for non-bijective transformations as
$
\rho(T) := \frac{I(X'; Y)}{I(X; Y)} \in (0, 1],
$
to measure the fraction of task-relevant information preserved by $T$. The Data Processing Inequality guarantees $\rho(T) \leq 1$ with strict inequality for non-invertible $T$.
Our key insight here is that such transformations degrade model performance through two distinct mechanisms: (1) \textbf{Variance Inflation:} information reduction increases the variance of parameter estimates, requiring more samples to achieve equivalent statistical precision.
(2) \textbf{Optimal Loss Shift:} discarded information creates an irreducible performance gap: a hard limit that no amount of data can overcome. Specifically, for (1), information loss inflates the variance of learning. For unbiased estimators under corrupted data, variance scales inversely with usable information:
$
\text{Var}(\hat{\theta}_{\text{corrupted}}) \propto \frac{1}{D \cdot \rho(T)^\nu},
$
where $\nu > 0$ depends on the corruption type, which controls how severely information loss penalizes statistical efficiency: larger $\nu$ means variance grows faster as $\rho(T)$ decreases. In our empirical evaluations, $\nu$ falls within the range $(0, 1)$ (e.g., $\nu \approx 0.19$ for next-token prediction tasks), reflecting the influence of the redundancy and noise present in real-world data.
Recent work \cite{subramanyam2025scaling} has discovered similar mechanisms. However, it doesn't capture mechanism (2), where the optimal pretraining loss on corrupted data is strictly worse than on clean data.

\begin{lemma}[Information Ceiling Bound]
\label{lemma:ceiling}
Let $L^*_X$ and $L^*_{X'}$ denote optimal pretraining losses for predicting $Y$ from $X$ and $X' = T(X)$ respectively. Then:
\begin{equation}
L^*_{X'} - L^*_X \geq \Phi\left(I(X;Y) - I(X';Y)\right) = \Phi\left(I(X;Y)(1 - \rho(T))\right)
\end{equation}
where $\Phi: \mathbb{R}^+ \to \mathbb{R}^+$ is a monotonically increasing function with $\Phi(0) = 0$.
\end{lemma}
The proof of Lemma \ref{lemma:ceiling} can be found at Appendix \ref{app:invariant_proofs}. It states that when applying a non-bijective transformation, the best achievable loss on the transformed data is strictly worse than on the original data, which creates an information ceiling bound dependent on $\rho$ even as $D\rightarrow \infty$. We now propose an updated scaling law that incorporates both mechanisms.

\begin{proposition}[Information-Resolution Scaling Law]
\label{prop:info_capacity}
Let $T$ be a non-bijective transformation with information resolution $\rho(T) = I(T(X);Y)/I(X;Y) \in (0,1]$. The expected test loss follows:
\begin{equation}
L(N, D, \rho) = \frac{A}{N^\alpha} + \underbrace{\frac{B}{D^\beta} \cdot \rho(T)^{-\nu}}_{\text{Variance Inflation} ~\phi_{\text{VI}}} + E + \underbrace{\kappa(1 - \rho(T))^\mu}_{\text{Optimal Loss Shift}~\phi_{\text{LS}}},
\label{eq:info_capacity}
\end{equation}
where $\nu > 0$ governs the variance inflation term $\phi_{\text{VI}}=(B/D^\beta) \cdot \rho(T)^{-\nu}$, which means when $\rho(T) < 1$, more samples are needed to achieve the same estimation precision; $E$ is the original irreducible loss on clean data; $\phi_{\text{LS}}=\kappa(1-\rho(T))^\mu$ is the shift of optimal loss, where $\kappa > 0$ is the scaling coefficient, and $\mu \in (0,1]$ governs how quickly the optimal loss shift can be reached. If the transformation is deterministic, then $\mu=1$ according to Lemma \ref{lemma:ceiling}; for stochastic corruptions, $\mu < 1$ because averaging over samples could partially recover obscured information.
\end{proposition}
Equation~\eqref{eq:info_capacity} provides a meta-scaling law: given a source calibration, it predicts the loss surface under changes in information resolution. Table~\ref{tab:comparison} contrasts our formulation with other scaling law formulations under data degradation. Specifically, our formulation satisfies two boundary conditions: (1) when $\rho = 1$ (bijective transformation), it reduces to the standard Chinchilla law; (2) when $\rho \to 0$, the ceiling term dominates and $L \to A/N^\alpha + E + \kappa$. A critical distinction is that, as the amount of data increases, prior quality-aware scaling laws converge to the same loss $E$ regardless of quality, while our formulation predicts a \textit{$\rho$-dependent loss}. This aligns with information-theoretic limits: no amount of degraded data can recover information that was discarded.

\begin{figure*}[t]
    \begin{minipage}{\textwidth}
    \centering
    \begin{tabular}{@{\hspace{-3.8ex}} c @{\hspace{-2.4ex}} c @{\hspace{-1.5ex}} c @{\hspace{-1.5ex}}}
        \begin{tabular}{c}
        \includegraphics[width=.31\textwidth]{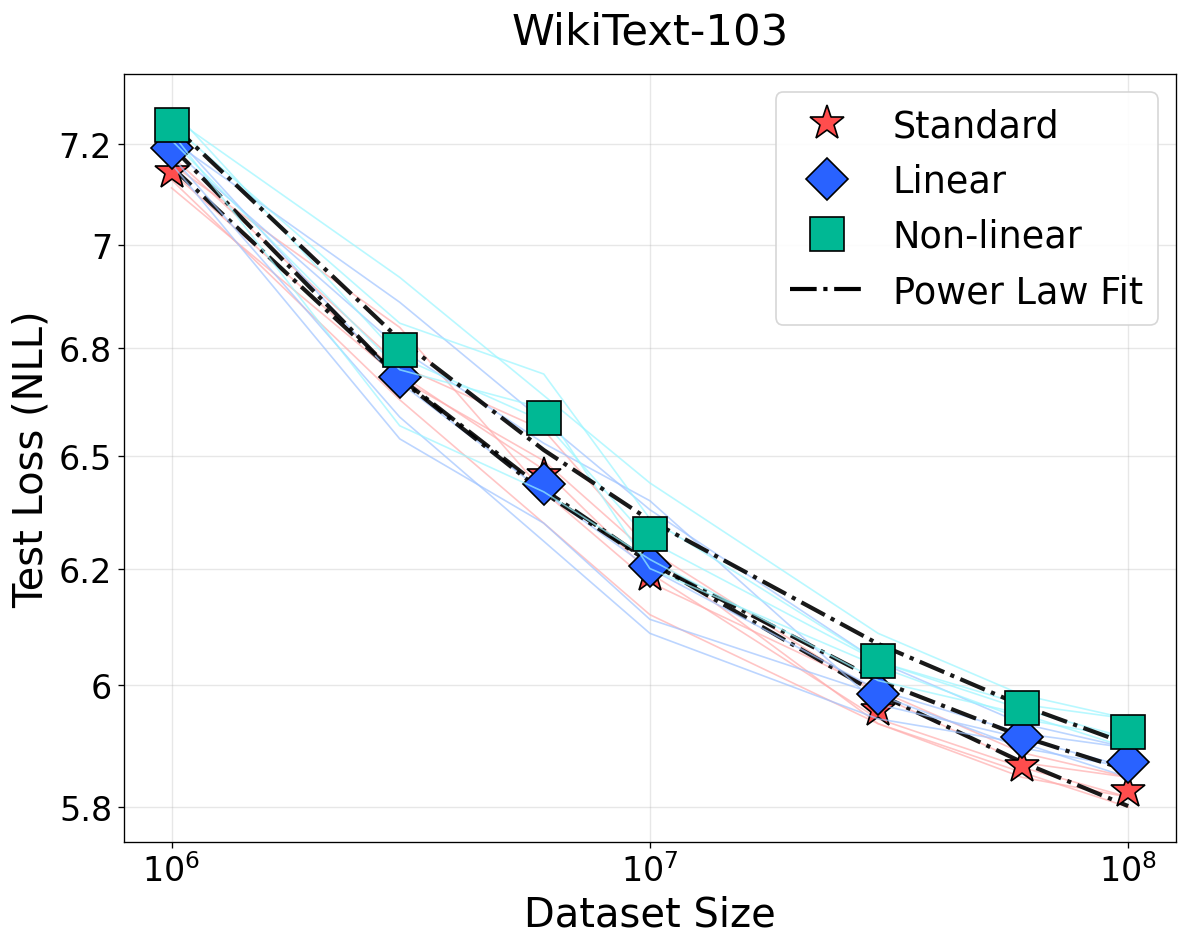}
        \\
        {\small{(a) Bijective}}
        \end{tabular} & 
        \begin{tabular}{c}
        \includegraphics[width=.31\textwidth]{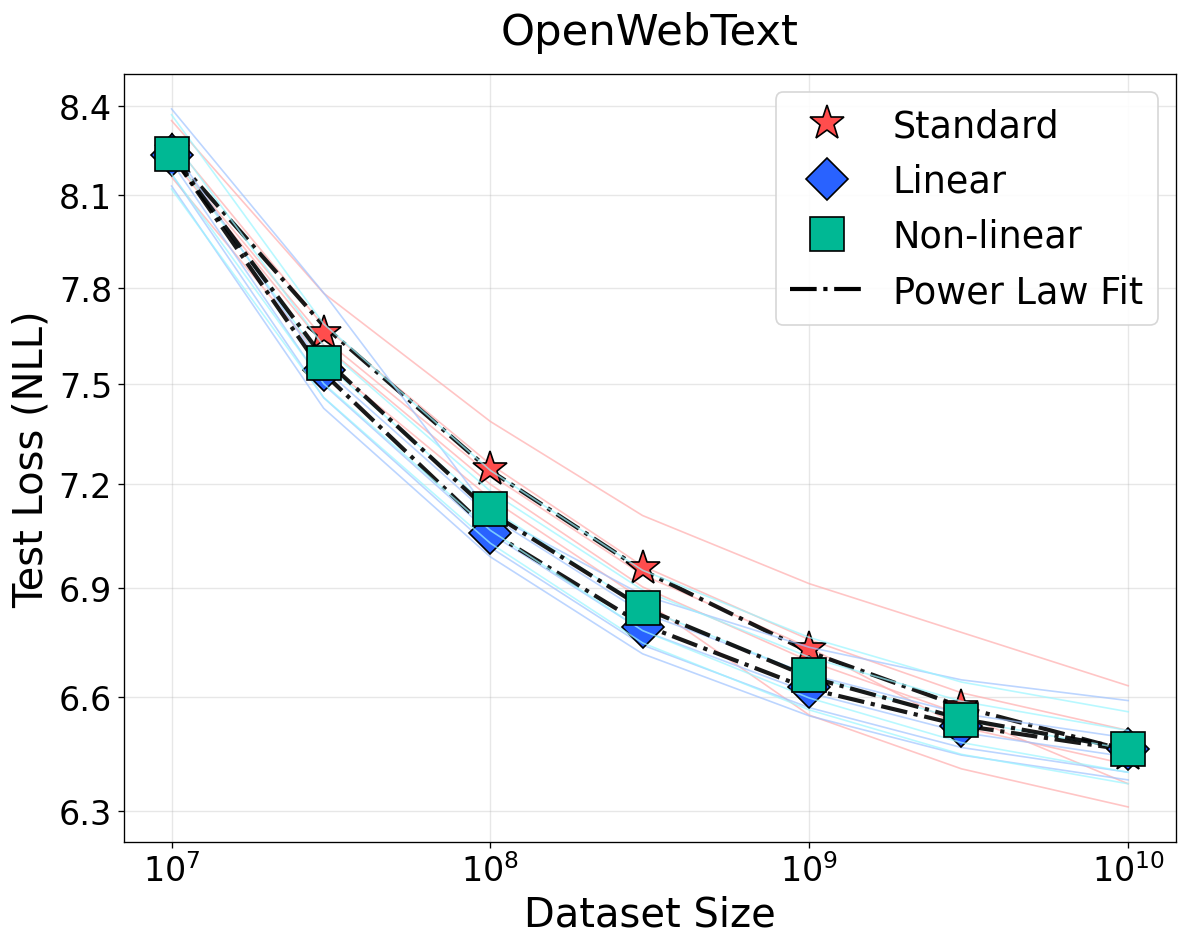} 
        \\
        {\small{(b) Bijective}}
        \end{tabular} &
        \begin{tabular}{c}
        \includegraphics[width=.31\textwidth]{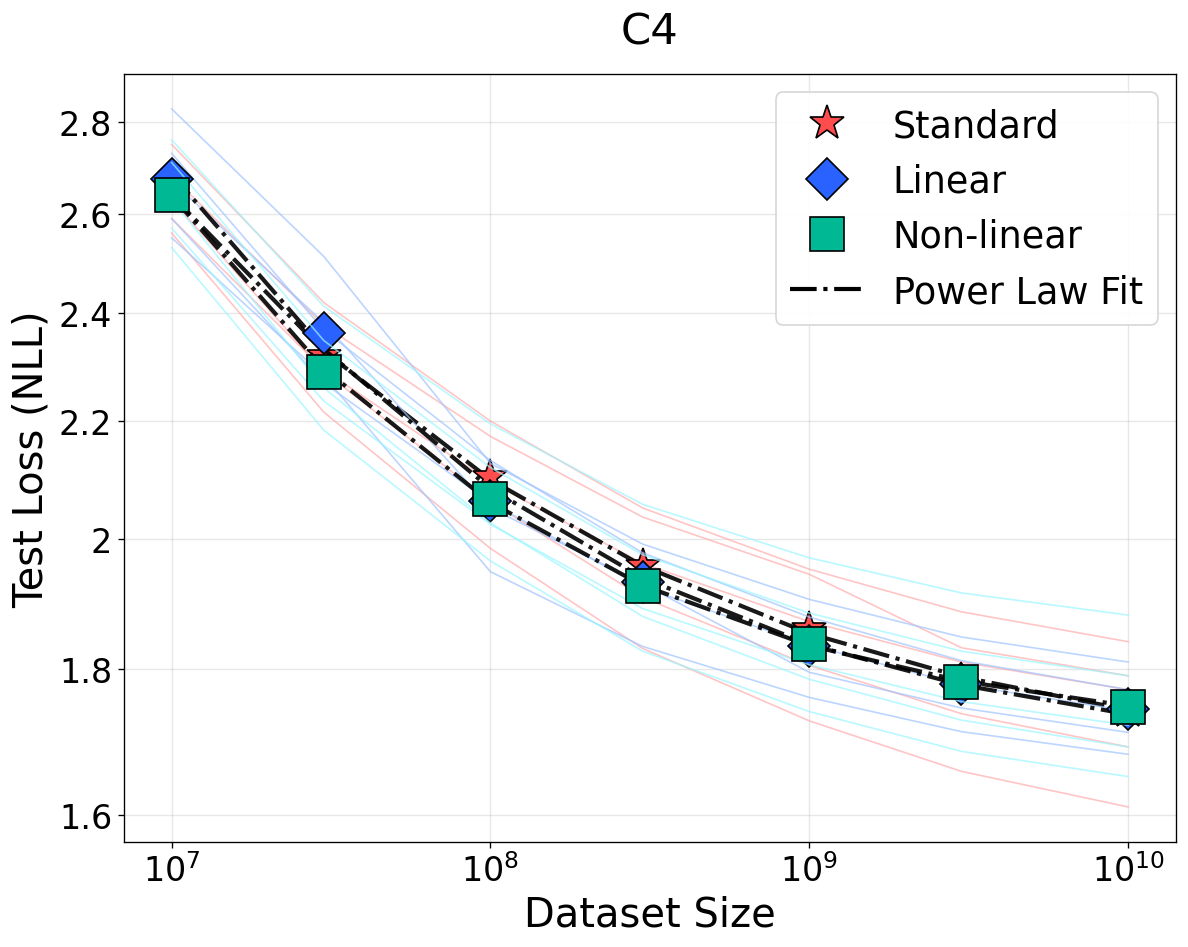} 
        \\
        {\small{(c) Bijective}}
        \end{tabular} \\
        \begin{tabular}{c}
        \includegraphics[width=.31\textwidth]{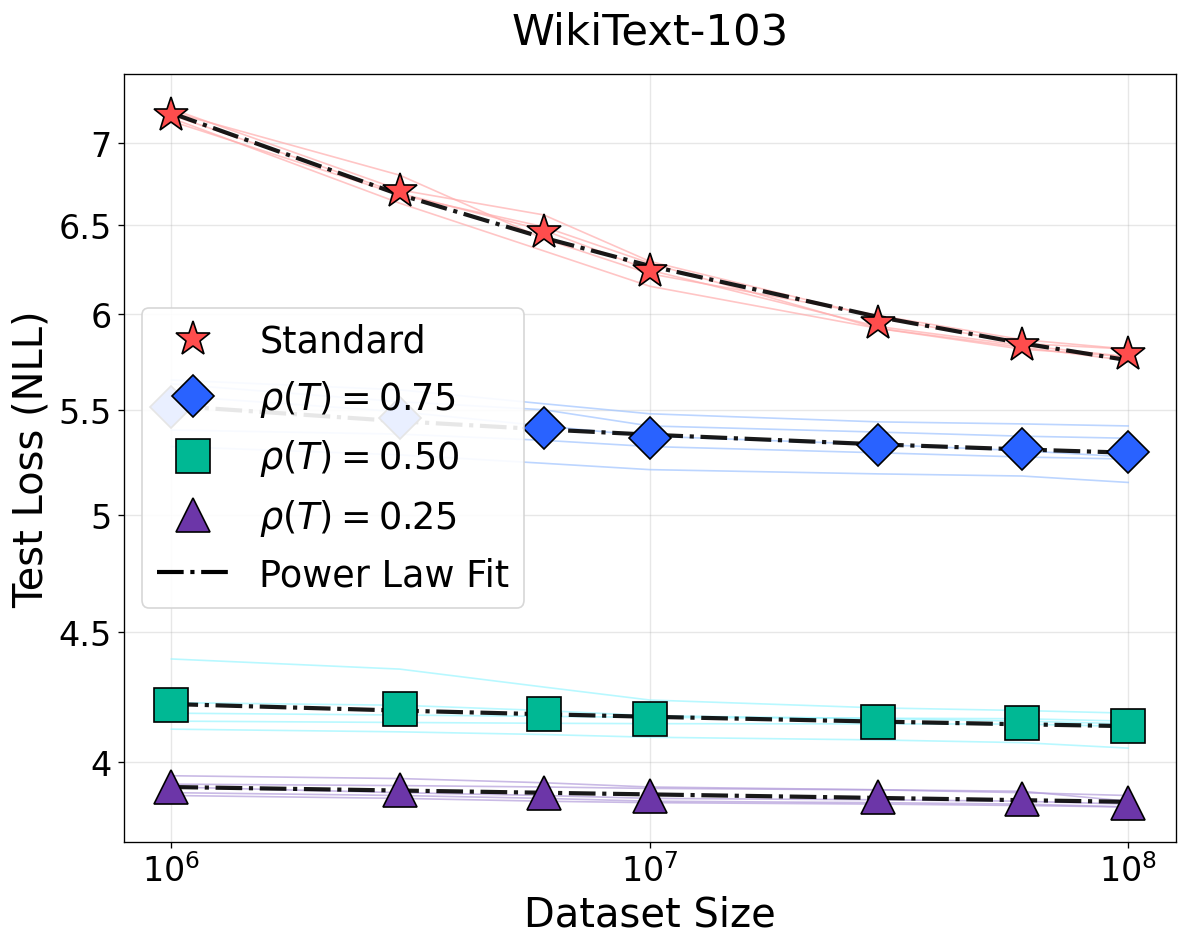}
        \\
        {\small{(d) Non-Bijective}}
        \end{tabular} & 
        \begin{tabular}{c}
        \includegraphics[width=.31\textwidth]{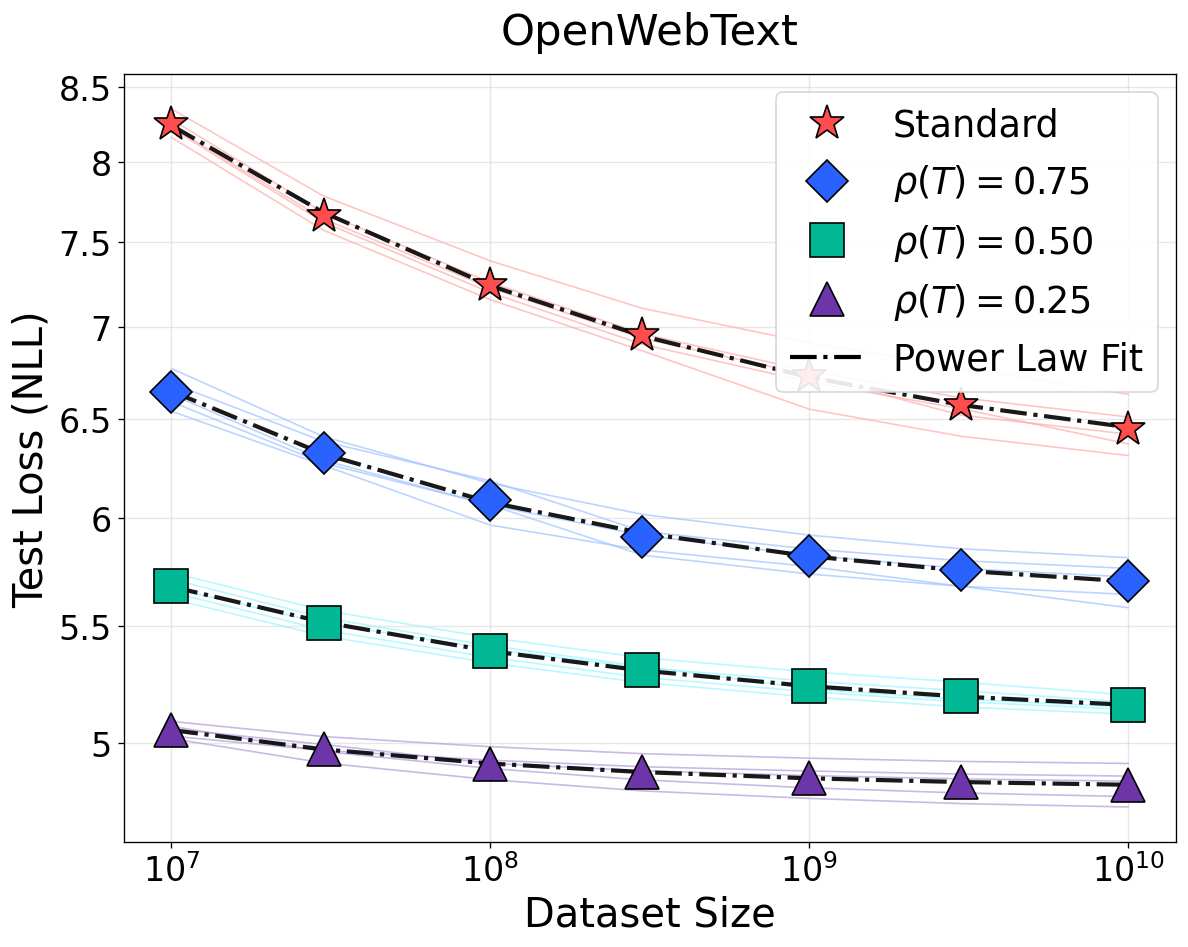} 
        \\
        {\small{(e) Non-Bijective}}
        \end{tabular} &
        \begin{tabular}{c}
        \includegraphics[width=.31\textwidth]{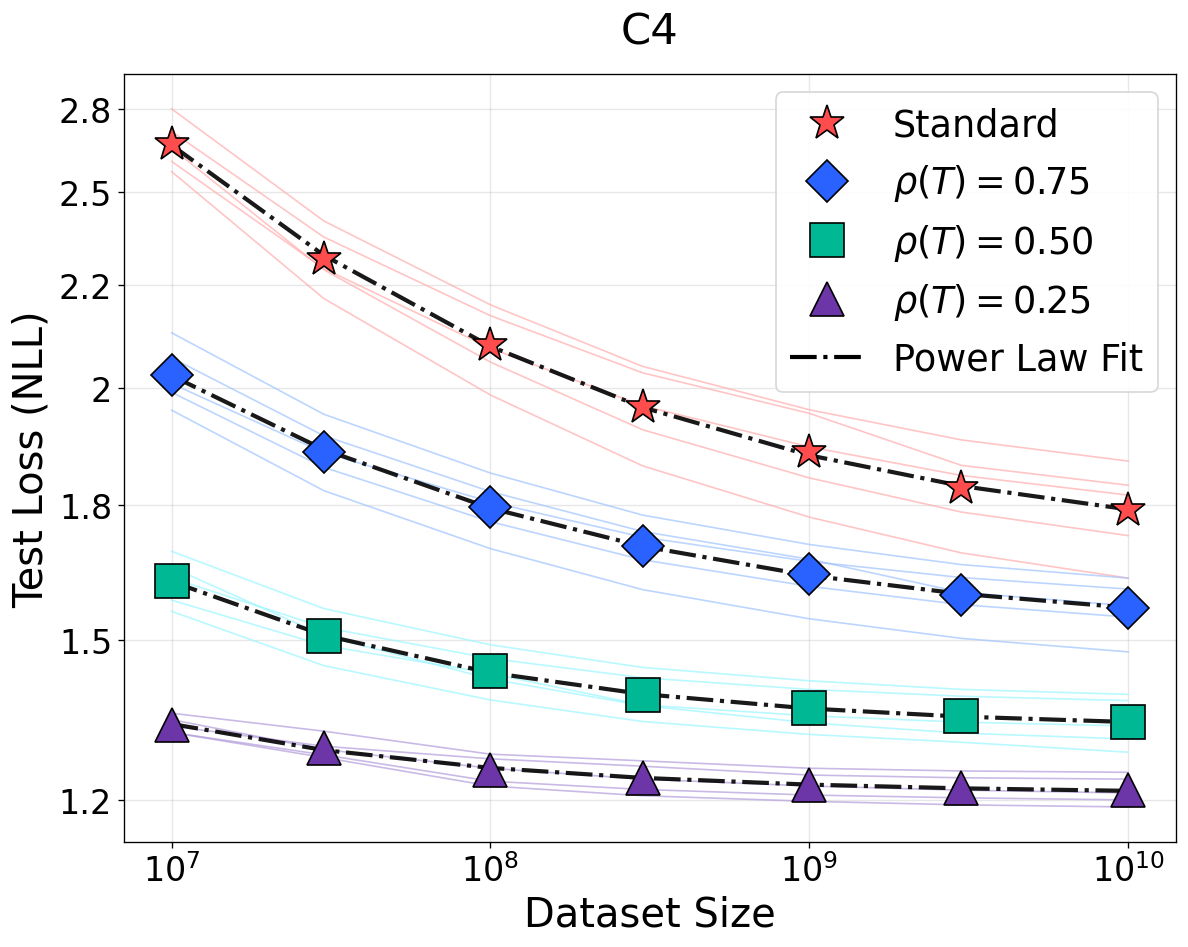} 
        \\
        {\small{(f) Non-Bijective}}
        \end{tabular} \\
        \end{tabular}
    \end{minipage}
    \caption{NSLs for language modeling on the WikiText (left), OpenWebText (middle), and C4 (right) datasets. The first row presents results under bijective transformations, the second row corresponds to non-bijective transformations (quantization). In the figures, the thin colored lines represent the test loss curves for individual parameter sizes, while the scaling law estimates (black dotted lines) are fitted using the average test loss values across parameter sizes. Results on negative log-likelihood (NLL) indicate that bijective transformations do not alter the scaling behavior; in contrast, as the degree of information loss increases, the scaling behavior progressively weakens.}
    \label{fig:lm_results}
\end{figure*}

\section{Primary Empirical Observations}
\label{sec:exp}
We conduct comprehensive empirical evaluations to validate our hypothesis and proposed information–resolution scaling law framework, aiming to answer the following questions: (\textbf{RQ1}) Do standard NSLs remain invariant under bijective transformations? (\textbf{RQ2}) How do non-bijective transformations that alter the information resolution of a dataset affect NSLs, and how effectively does our framework capture these effects? (\textbf{RQ3}) How can these insights be leveraged in real-world cross-domain applications? We evaluate scaling behavior across diverse modality benchmarks, including language modeling, vision, speech, and medical records. Our implementation is available at \url{https://github.com/aaronhan223/scaling_law}.

\begin{figure*}[t]
    \begin{minipage}{\textwidth}
    \centering
    \begin{tabular}{@{\hspace{-3.8ex}} c @{\hspace{-2.4ex}} c @{\hspace{-1.5ex}} c @{\hspace{-1.5ex}}}
        \begin{tabular}{c}
        \includegraphics[width=.31\textwidth]{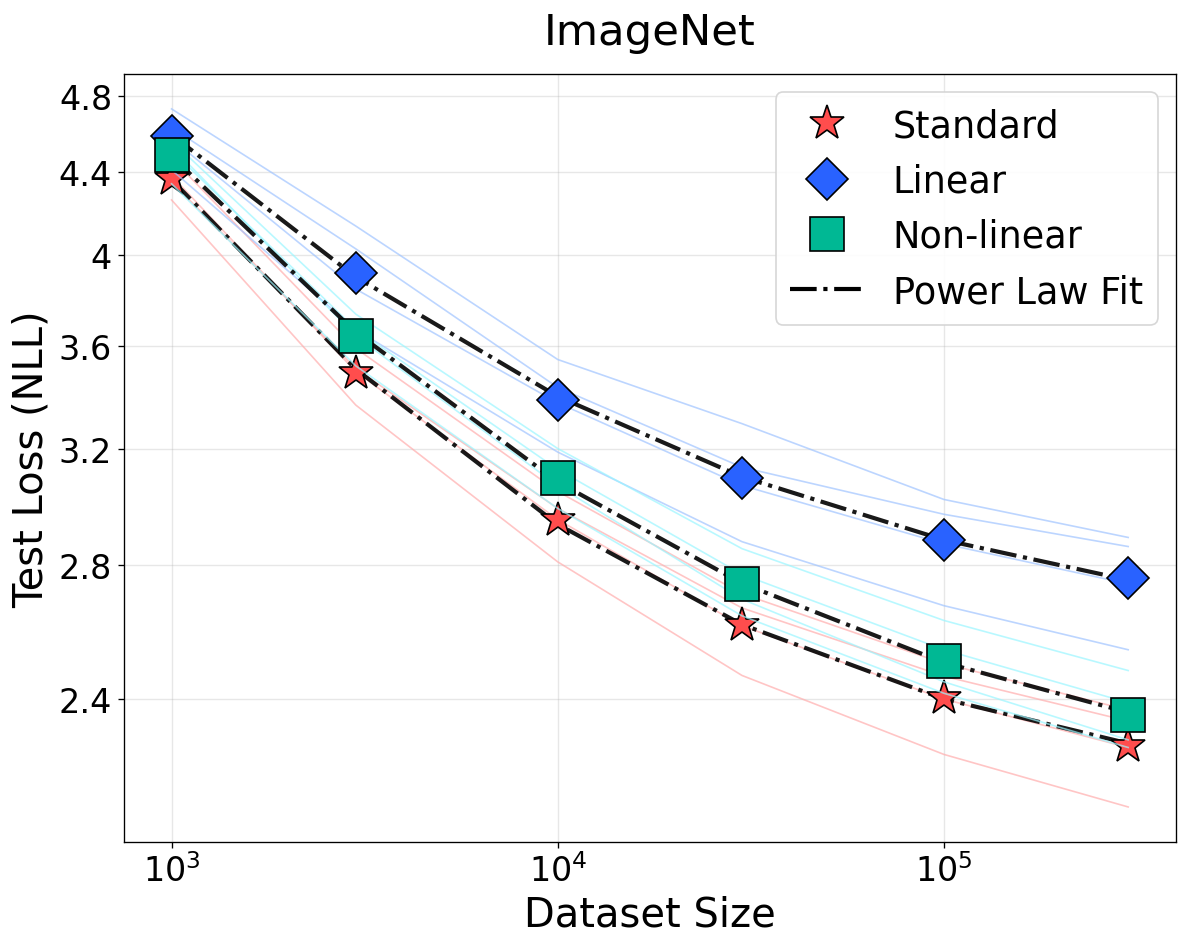}
        \\
        {\small{(a) Bijective}}
        \end{tabular} & 
        \begin{tabular}{c}
        \includegraphics[width=.31\textwidth]{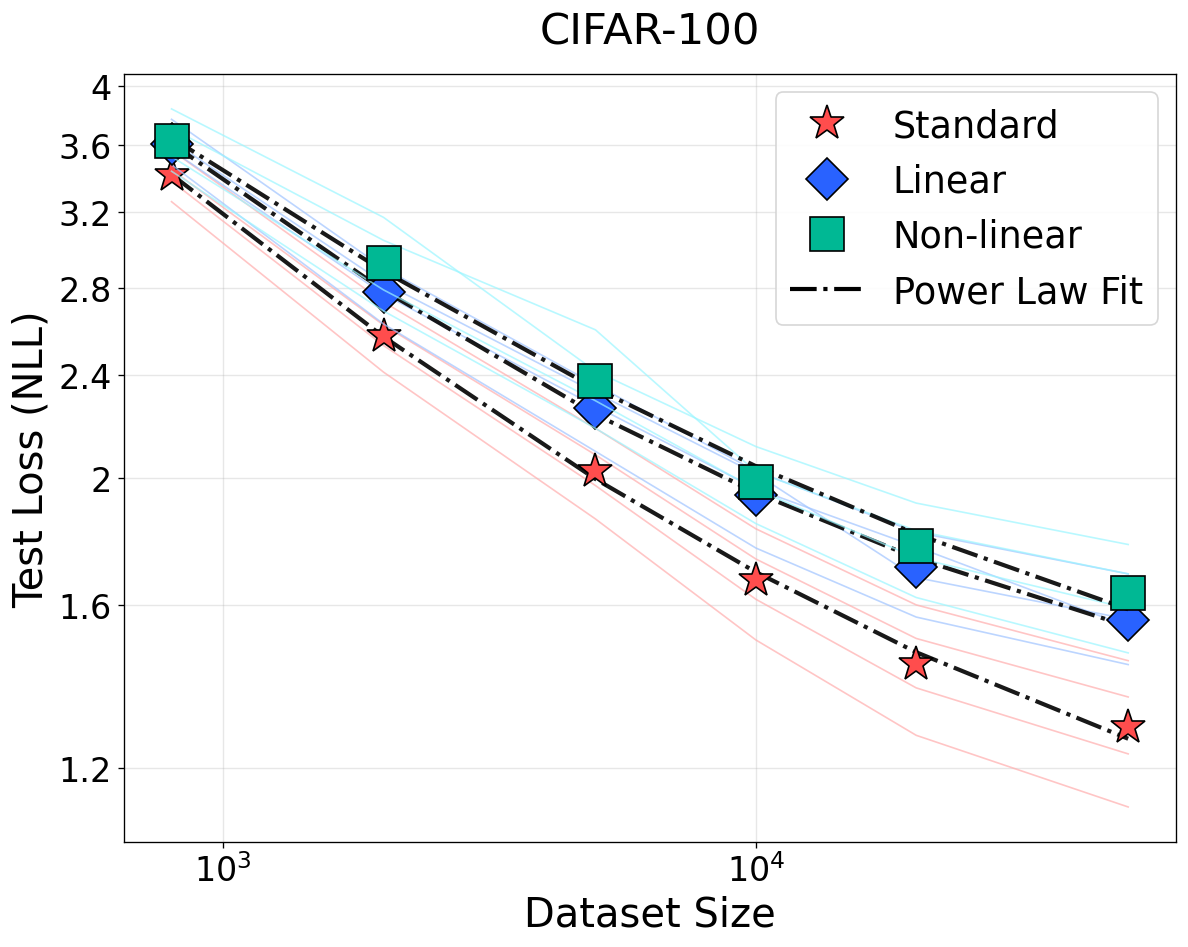} 
        \\
        {\small{(b) Bijective}}
        \end{tabular} &
        \begin{tabular}{c}
        \includegraphics[width=.31\textwidth]{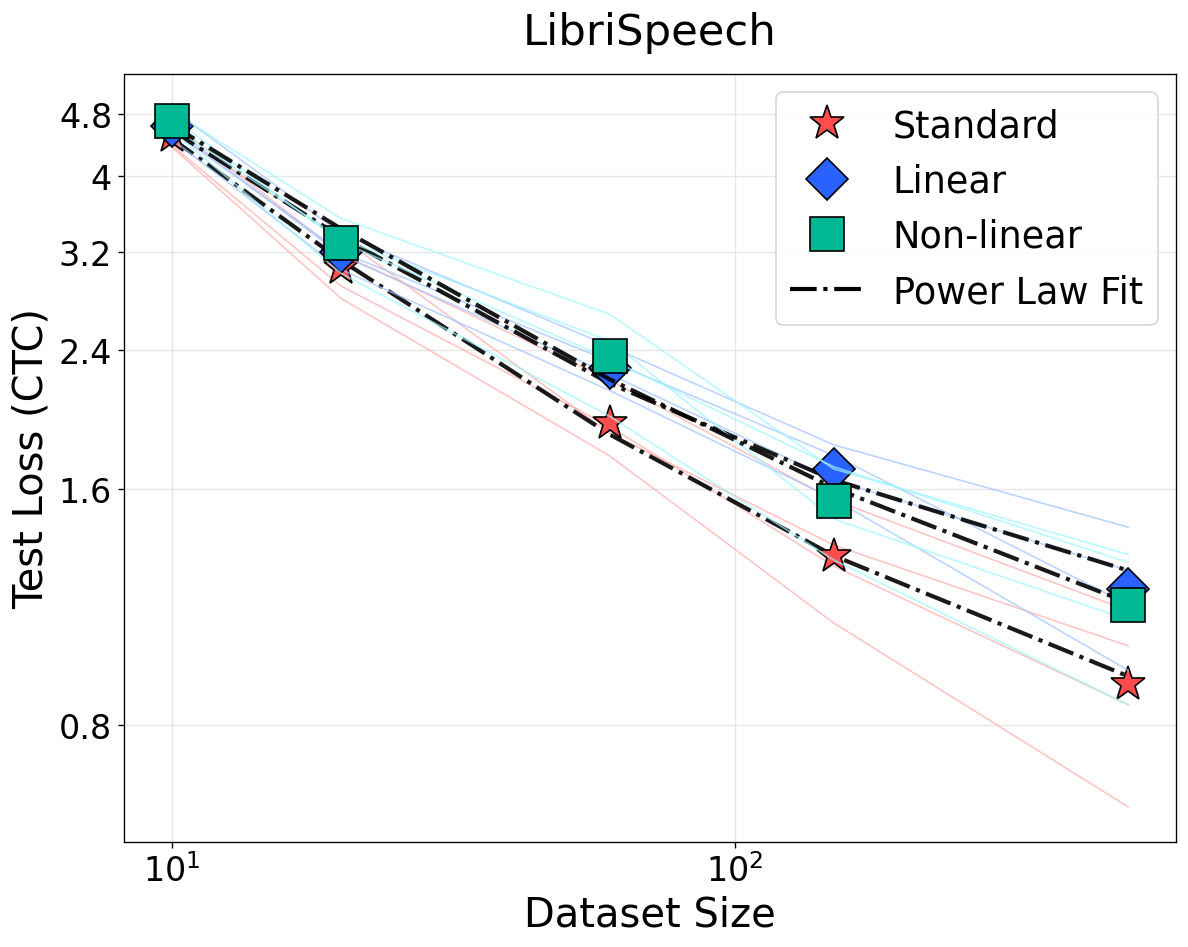} 
        \\
        {\small{(c) Bijective}}
        \end{tabular} \\
        \begin{tabular}{c}
        \includegraphics[width=.31\textwidth]{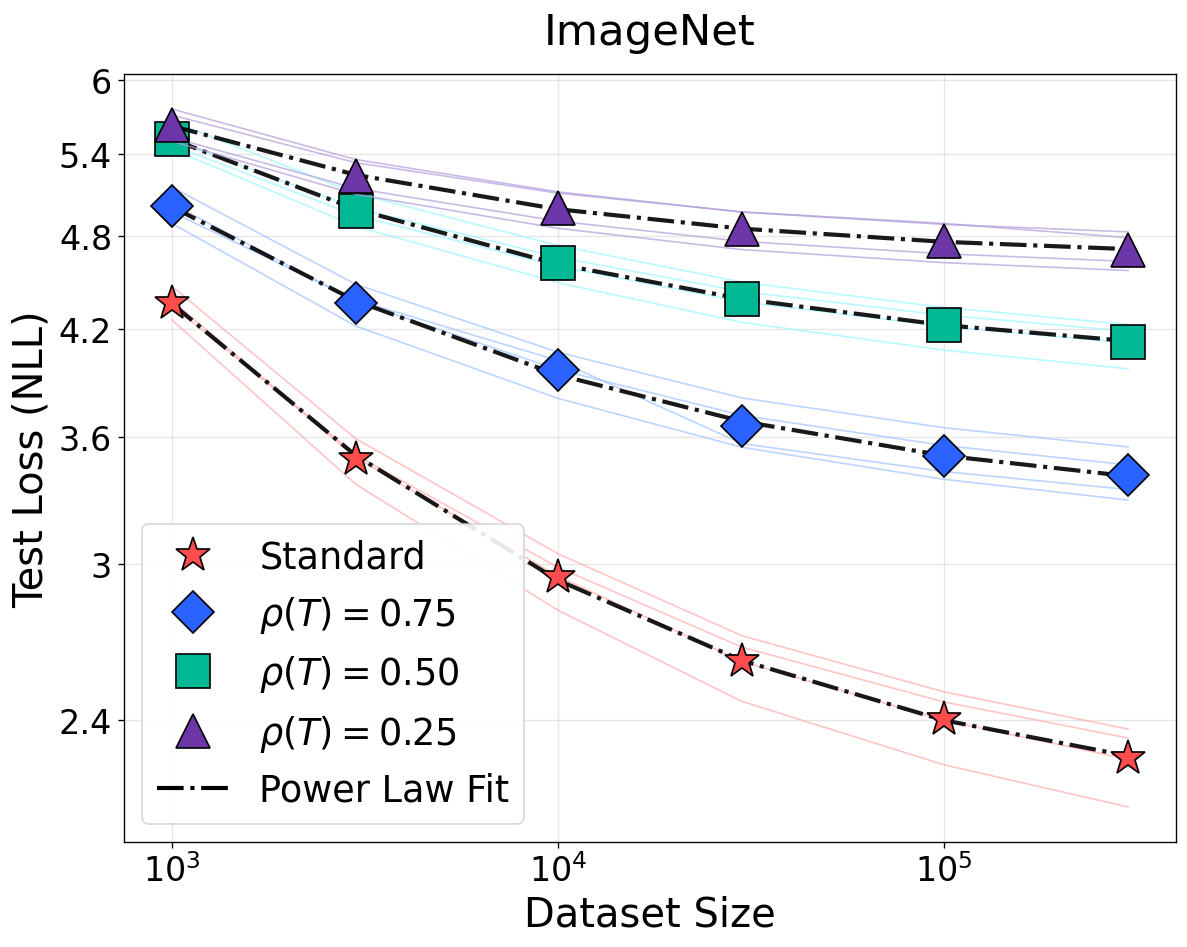}
        \\
        {\small{(d) Non-Bijective}}
        \end{tabular} & 
        \begin{tabular}{c}
        \includegraphics[width=.31\textwidth]{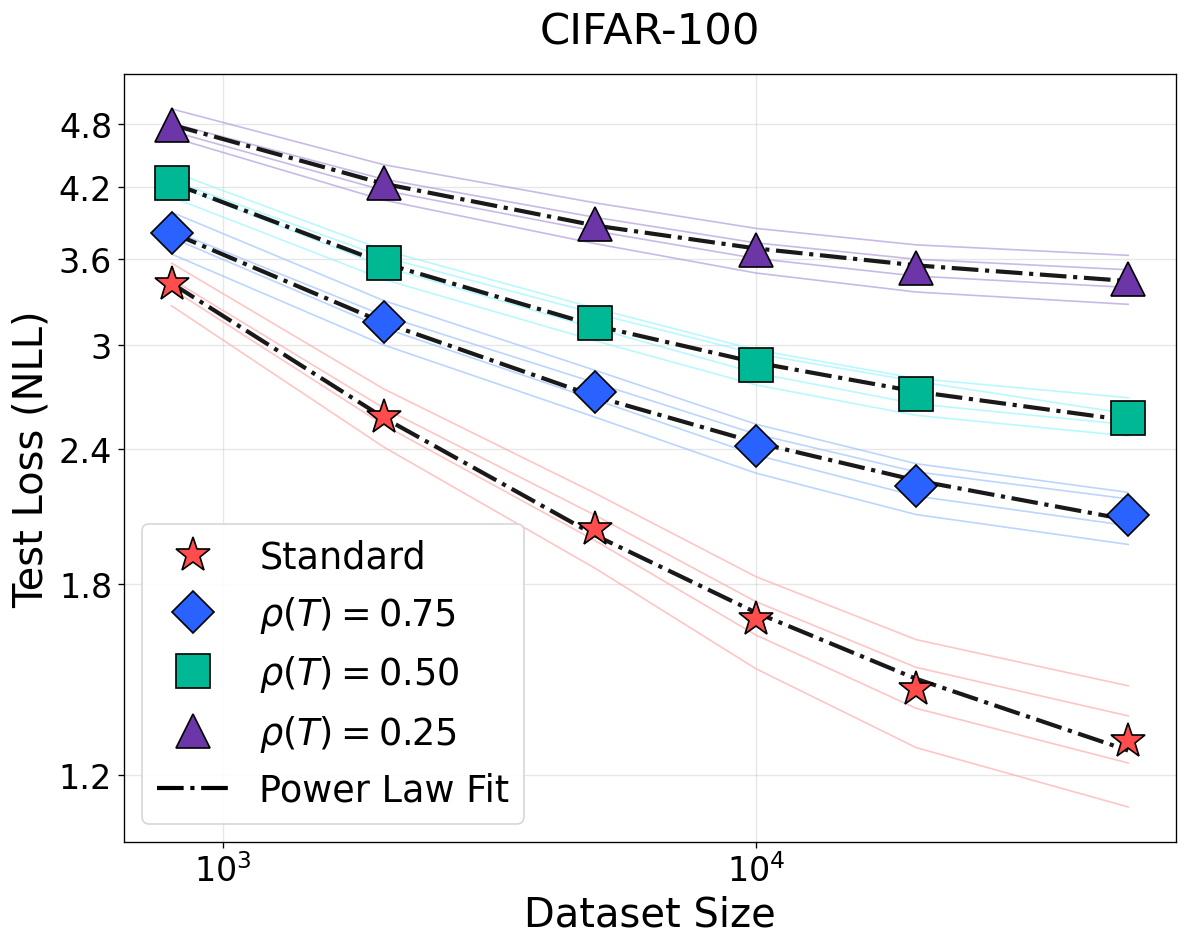} 
        \\
        {\small{(e) Non-Bijective}}
        \end{tabular} &
        \begin{tabular}{c}
        \includegraphics[width=.31\textwidth]{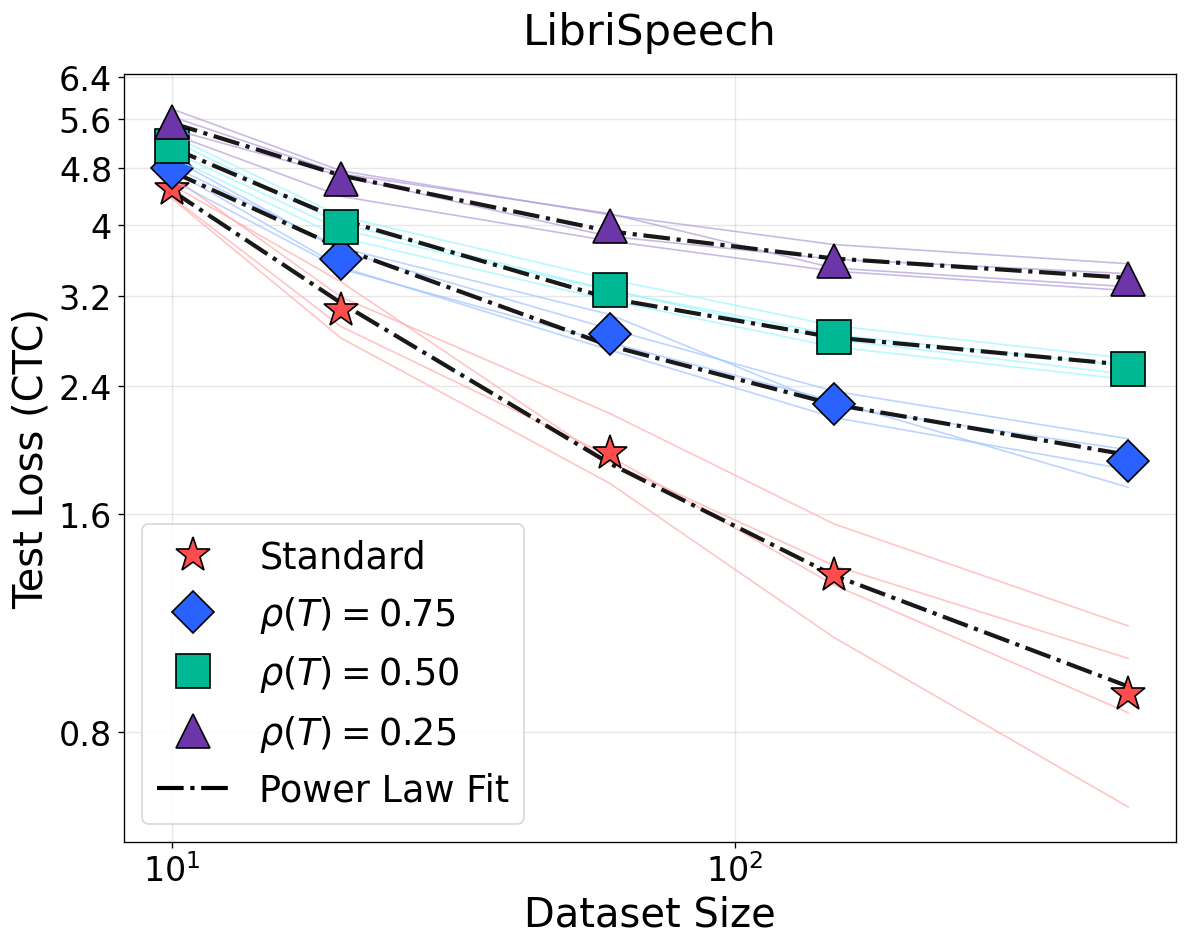} 
        \\
        {\small{(f) Non-Bijective}}
        \end{tabular} \\
        \end{tabular}
    \end{minipage}
    \caption{NSLs for vision and speech tasks on the ImageNet-1K (left column), CIFAR-100 (middle column), and LibriSpeech (right column) datasets. Performance is evaluated using negative log-likelihood (NLL) and Connectionist Temporal Classification (CTC) loss. Similarly, the non-bijective transformation (low-rank) shows a much stronger influence on scaling behavior.}
    \label{fig:other_results}
\end{figure*}

\textbf{Experimental Setup.}
We evaluate our framework across four modalities. Language: causal Llama-3-style Transformers \citep{dubey2024llama} with the GPT-2 tokenizer (context 1024) trained on WikiText-103 \citep{merity2016pointer}, OpenWebText \citep{Gokaslan2019OpenWeb}, and C4 \citep{dodge2021documenting} for next-token prediction and translation. Vision: ViTs \citep{dosovitskiy2020image} on ImageNet-1K \citep{deng2009imagenet} and ResNets \citep{he2016deep} on CIFAR-100 \citep{krizhevsky2009learning}. Speech: Wav2Vec2 \citep{baevski2020wav2vec} on LibriSpeech \citep{panayotov2015librispeech}. Clinical time-series: MIMIC-IV \citep{johnson2023mimic} ICU vitals and labs. Each scaling-law fit uses 4–6 model capacities and 5–6 dataset sizes. We probe two bijective transformation families: linear (full-rank random matrix on token IDs) and nonlinear (affine coupling layers from the normalizing flows literature \cite{dinh2016density}), and two non-bijective families: quantization (Eq. \eqref{eq:quant}) and low-rank projection  (Eq. \eqref{eq:low_rank}) of embeddings at multiple intensity levels. Full details are deferred to Appendix~\ref{app:exp_details}.

\begin{figure*}[h]
    \begin{minipage}{\textwidth}
    \centering
    \begin{tabular}{@{\hspace{-3.8ex}} c @{\hspace{-2.4ex}} c @{\hspace{-1.5ex}} c @{\hspace{-1.5ex}}}
        \begin{tabular}{c}
        \includegraphics[width=.31\textwidth]{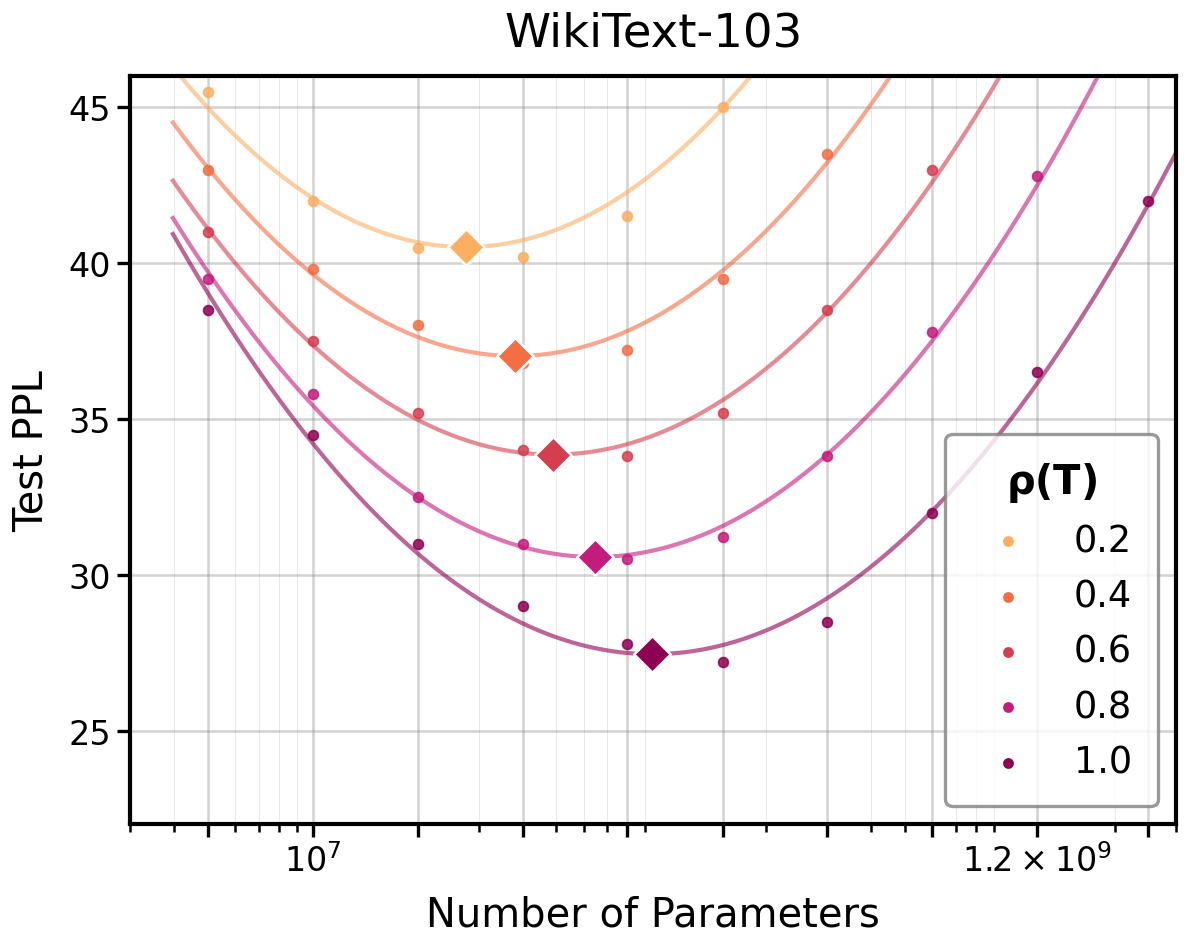}
        \\
        {\small{(a)}}
        \end{tabular} & 
        \begin{tabular}{c}
        \includegraphics[width=.31\textwidth]{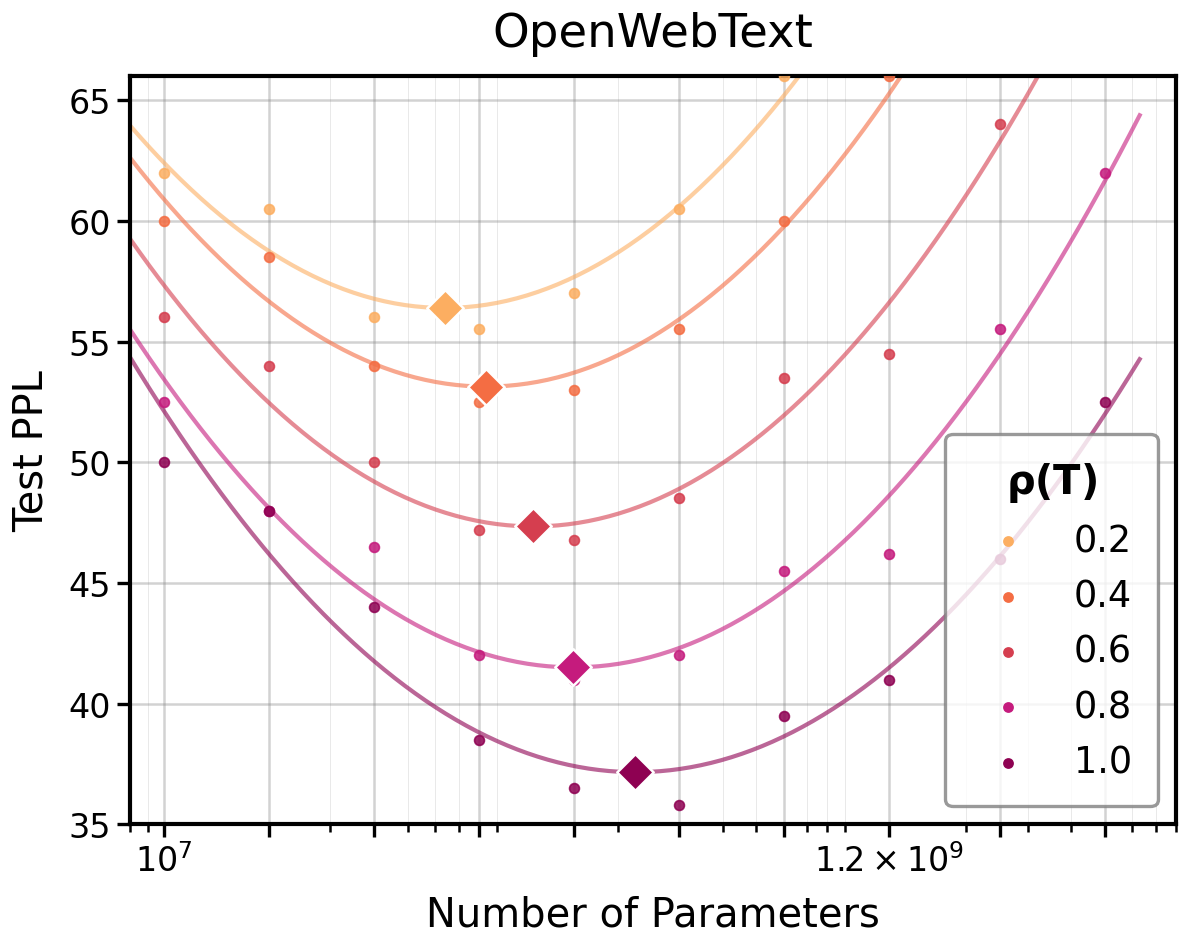} 
        \\
        {\small{(b)}}
        \end{tabular} &
        \begin{tabular}{c}
        \includegraphics[width=.31\textwidth]{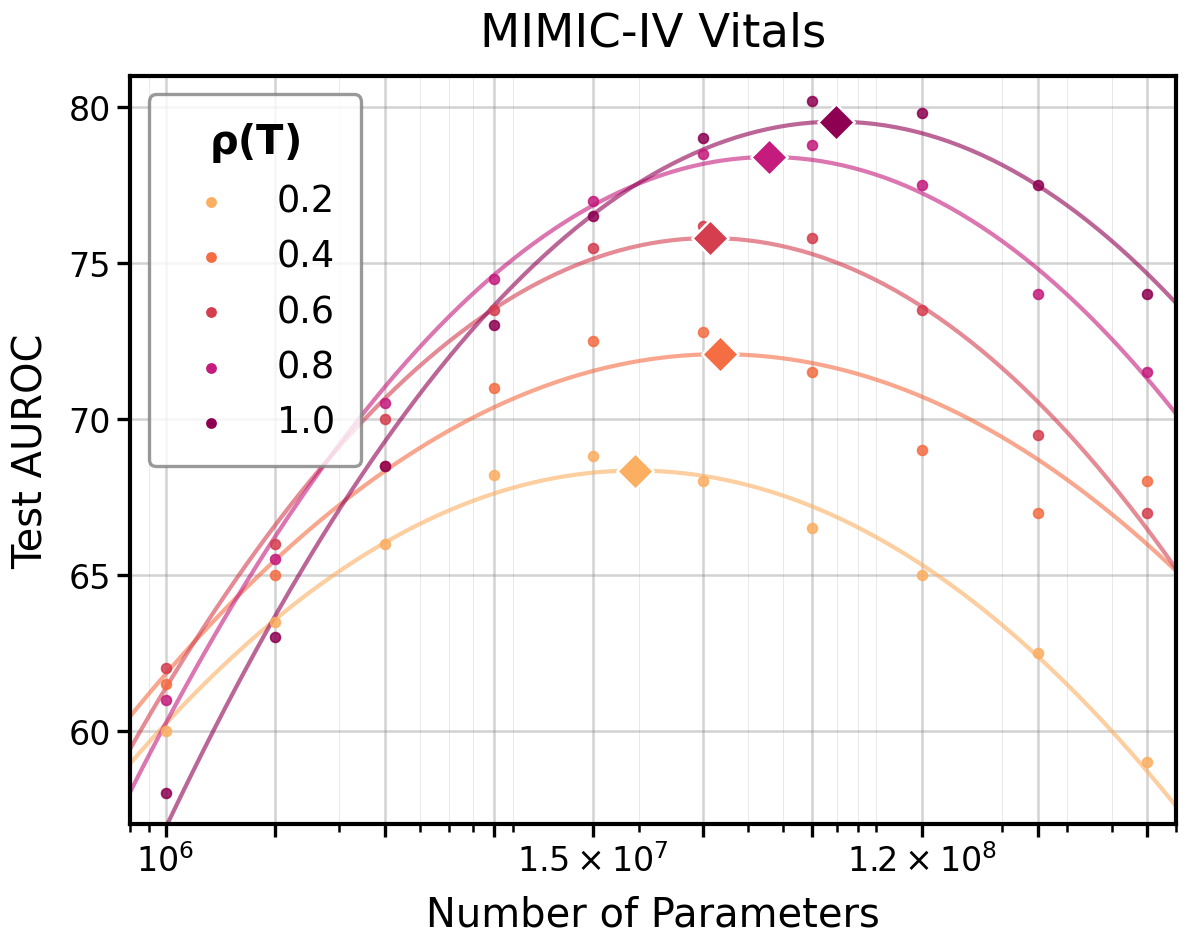} 
        \\
        {\small{(c)}}
        \end{tabular} \\
        \end{tabular}
    \end{minipage}
    \caption{Optimal model capacity as a function of $\rho(T)$. Each curve is an iso-$\rho$ slice: model capacity is swept across ten scales at fixed dataset size and fixed $\rho(T)$, and we plot test loss on the y-axis. The U-shape roughly reflects the underfitting-to-overfitting tradeoff; the minimum of each curve marks the compute-optimal capacity $N^*(\rho)$. As $\rho(T)$ decreases, the optimal model size shrinks, consistent with our prediction from Eq.~(\ref{eq:opt_capacity}). $\rho(T)$ is varied via embedding quantization for the language tasks and quantile binning of clinical time-series for MIMIC-IV.}
    \label{fig:optimal_params}
\end{figure*}

\textbf{NSLs Under Various Transformations.}
Figure~\ref{fig:lm_results} (a)–(c) shows that NSLs remain largely invariant under bijective transformations, with linear and nonlinear re-encodings tracing nearly identical scaling curves to the untransformed baseline. In contrast, Figure~\ref{fig:lm_results} (d)–(f) shows that progressively reducing $\rho(T)$ via quantization weakens the scaling curves and induces a clear saturation effect at each $\rho$ level, with the asymptote tracking the $\rho$-dependent loss floor predicted by Eq.~\eqref{eq:info_capacity}. Note that transformations such as quantization map multiple token embeddings to the same token ID, effectively reducing the overall vocabulary size. This leads to a decrease in NLL, while perplexity correspondingly increases.
Figure~\ref{fig:other_results} shows that the same qualitative pattern holds for vision (ImageNet, CIFAR) and speech (LibriSpeech): bijective transformations of pixel or spectrogram inputs preserve scaling exponents within fitting noise, while non-bijective transformations shift the asymptotic floor in a $\rho$-dependent manner. Notably, the magnitude of degradation is smaller than in language modeling: classification tasks tolerate substantial information loss before scaling visibly suffers, likely because the prediction target is a coarse label rather than a high-entropy next-token distribution, and because these tasks rely less on long-range sequential context that quantization and projection most severely disrupt. The redundancy of natural images and the slot-like structure of speech tokens further mean that much of the discarded information is task-irrelevant. In this moderate-resolution regime, both variance inflation and Bayes-risk shift contribute non-negligibly to the observed scaling behavior.

\begin{algorithm}[t]
\caption{Procedure of Cross-domain Scaling Law Prediction}
\label{alg:predict_scaling}
\begin{algorithmic}[1]
\Require Source-domain losses $\{L_s(N, D, \rho)\}$ across varying $\rho$; transformation $T$.
\Ensure Predicted target-domain exponents $\hat{\beta}_t$ and irreducible loss $\hat{E}_t$.
\State \textbf{Estimate info-resolution parameters from source.} Fit Eq.~\eqref{eq:info_capacity} to source-domain losses across different levels of $\rho$ to obtain the source-domain scaling law and transformation-type parameters $\hat{\nu}, \hat{\mu}, \hat{\kappa}$, together with $A_s, \alpha_s, B_s, \beta_s, E_s$.
\State \textbf{Compute target information resolution analytically.} Estimate $\rho(T)$ via the transformation-specific formula (e.g. Eqs. \eqref{eq:quant}-\eqref{eq:noise}).
\State \textbf{Predict target loss surface.} Substitute $\rho(T)$ and $\hat{\nu}, \hat{\mu}, \hat{\kappa}$ into the info-resolution scaling law:
$
    L_t(N, D) \;=\; \frac{A_s}{N^{\alpha_s}} \;+\; \frac{B_s}{D^{\beta_s}} \cdot \rho(T)^{-\hat{\nu}} \;+\; E_s \;+\; \hat{\kappa}\bigl(1 - \rho(T)\bigr)^{\hat{\mu}}.
$
\State \textbf{Extract effective Chinchilla-form exponents.} Fit
$
    L_t \;=\; \frac{A_t}{N^{\hat{\alpha}_t}} \;+\; \frac{B_t}{D^{\hat{\beta}_t}} \;+\; \hat{E}_t
$
to the analytically predicted surface to obtain $\hat{\beta}_t$ and $\hat{E}_t$.
\State \Return $\hat{\beta}_t,\ \hat{E}_t$.
\end{algorithmic}
\end{algorithm}

\section{Predicting Scaling Across Domains}

\subsection{Estimating Optimal Model Capacity}
Determining the optimal model capacity for a given dataset is one of the key insights provided by scaling laws. Here, we investigate how changes in information resolution affect this optimum, offering practical guidance for model design under varying levels of distortion. Our information-resolution scaling law admits a closed-form prediction for this quantity. Under compute-optimal allocation $C = 6ND$~\citep{hoffmann2022training}, minimizing Eq.~(\ref{eq:info_capacity}) over $N$ yields
\begin{equation}
    N^*(C, \rho) \;=\; N^*(C, 1) \cdot \rho^{\,\nu/(\alpha + \beta)},
    \qquad
    D^*(C, \rho) \;=\; D^*(C, 1) \cdot \rho^{-\nu/(\alpha + \beta)},
    \label{eq:opt_capacity}
\end{equation}
where $N^*(C, 1)$ is the standard Chinchilla-optimal model size on clean data. The optimal model size therefore decreases as a power law in $\rho$, while the optimal token budget grows by the inverse factor.
To validate this prediction empirically, we train ten model capacities at each of several $\rho(T)$ levels on three benchmarks: language modeling on WikiText-103 and OpenWebText, and in-hospital mortality prediction on MIMIC-IV. Dataset size is held fixed so that capacity is the sole varying factor. For language modeling, $\rho(T)$ is varied via quantization of token embeddings (consistent with Figure~\ref{fig:lm_results}); for mortality prediction, via quantile binning of clinical time-series at increasing coarseness.\looseness=-1

Figure~\ref{fig:optimal_params} confirms a clean monotonic relationship between optimal capacity and $\rho(T)$ across all three tasks, with the empirical curves closely tracking the $\rho^{\,\nu/(\alpha+\beta)}$ form predicted by Eq.~(\ref{eq:opt_capacity}). The direction of the effect is theoretically expected: when each sample carries less usable information, the compute-optimal strategy shifts toward processing more samples with a smaller model, rather than allocating compute to a larger model that cannot recover information already destroyed by the transformation. What is more interesting is the magnitude of the shift: even moderate drops in $\rho$ produce substantial reductions in optimal capacity, suggesting that practitioners working with degraded or compressed data may be systematically over-provisioning model size.

\subsection{Cross-domain Scaling Prediction}
Forecasting scaling behavior in a new domain from a calibrated source law is the central application of our meta-scaling framework. This is most valuable in medical domain where privacy prevents access to abundant training data and compute resources are scarce.
We apply Algorithm \ref{alg:predict_scaling} to two prediction tasks that sit at different points along the applicability spectrum of our framework. \textbf{(i) Within dataset noise injection:} predicting scaling laws for 48-hour in-hospital mortality on MIMIC-IV time series under varying levels of injected noise, given the law fit on the clean signals. \textbf{(ii) Cross-corpus next-token prediction (more difficult):} predicting scaling laws for MIMIC-IV discharge summaries given the law fit on OpenWebText. The first shares dataset with the source and only varies a known corruption mechanism; the second requires assuming the corpus shift can itself be characterized as a reduction in $\rho(T)$. We discuss the key procedures (steps 2 and 3) below.

\begin{table}[t]
\centering
\caption{Predicted vs. empirical NSL parameters $\beta$ and $E$ across two cross-domain prediction tasks. ``Real'' denotes empirically fitted values; Quality-Aware~\citep{subramanyam2025scaling} and Sub-Optimal~\citep{chen2025revisiting} are baselines. Both baselines (approximately) predict $E_t = E_s$ and therefore systematically underestimate the loss-floor shift, while we recover $E$ within a small margin. See Appendix \ref{app:additional_results} for full results.\looseness=-1}
\label{tab:nsl_pred}
\renewcommand{\arraystretch}{0.75}
\setlength{\tabcolsep}{6pt}
\begin{tabular}{l| c| c| l |c |c}
\toprule
\textbf{Task / Transformation} & $\nu$ & $\rho(T)$ & \textbf{Method} & $\beta$ & $E$ \\
\midrule
\multirow{14}{*}{\shortstack[l]{\textbf{(i) Within Dataset:}\\ In-hospital Mortality\\ / Noise Injection}}
  & \multirow{14}{*}{0.12}
  & \multirow{4}{*}{0.75}
    & Real          & 0.2843 & 2.6533 \\
  & & & Quality-Aware~\citep{subramanyam2025scaling} & 0.2975 & 1.8314 \\
  & & & Sub-Optimal~\citep{chen2025revisiting}   & 0.2933 & 1.8285 \\
  & & & \textbf{Ours} & \textbf{0.2961} & \textbf{2.7842} \\
\cmidrule(lr){3-6}
  & & \multirow{4}{*}{0.50}
    & Real          & 0.2441 & 3.7833 \\
  & & & Quality-Aware~\citep{subramanyam2025scaling} & 0.2535 & 1.9768 \\
  & & & Sub-Optimal~\citep{chen2025revisiting}   & 0.2541 & 1.9344 \\
  & & & \textbf{Ours} & \textbf{0.2512} & \textbf{3.7814} \\
\cmidrule(lr){3-6}
  & & \multirow{4}{*}{0.25}
    & Real          & 0.1927 & 5.1371 \\
  & & & Quality-Aware~\citep{subramanyam2025scaling} & 0.2153 & 2.3907 \\
  & & & Sub-Optimal~\citep{chen2025revisiting}   & 0.2221 & 2.2415 \\
  & & & \textbf{Ours} & \textbf{0.1821} & \textbf{4.7839} \\
\midrule
\multirow{4}{*}{\shortstack[l]{\textbf{(ii) Cross Corpus:}\\ Next Token Prediction\\ / Quantization}}
  & \multirow{4}{*}{0.19} & \multirow{4}{*}{0.54}
    & Real          & 0.3254 & 4.4132 \\
  & & & Quality-Aware~\citep{subramanyam2025scaling} & 0.3335 & 3.3854 \\
  & & & Sub-Optimal~\citep{chen2025revisiting}   & 0.3619 & 3.4662 \\
  & & & \textbf{Ours} & \textbf{0.3212} & \textbf{4.0012} \\
\bottomrule
\end{tabular}
\end{table}

\textbf{Step 2: Computing $\rho(T)$.}
For \textit{(i) Noise Injection}, the noise distribution is known by construction, so $\rho(T)$ follows in closed form from Eq.~\eqref{eq:noise} by plugging the resulting signal-to-noise ratio into the SNR-based formula: no estimation required. For \textit{(ii) Corpus-Transfer}, we view MIMIC-IV discharge summaries as a \emph{sublanguage} of OpenWebText: same English grammar but a restricted vocabulary and formulaic clinical templates that collapse the diverse expressions of web text into a smaller set of recurring patterns, which acts as a quantization of OpenWebText's token distribution and lets Eq.~\eqref{eq:quant} apply with $V$ and $q$ corresponding to the source and target effective vocabularies. To estimate $\rho(T)$ in this case, we use a compression-based estimator $\hat{\rho}_{\text{compress}} = (C(\mathcal{D}')/|\mathcal{D}'|)/(C(\mathcal{D})/|\mathcal{D}|)$ with \texttt{gzip}~\citep{chang2024scaling}, which jointly captures vocabulary, frequency, syntax, local semantics, and repetition structure, where $C(\cdot)$ denotes compressed size (in bits), $|\mathcal{D}|$ denotes raw corpus size, and $C(\mathcal{D})/|\mathcal{D}|$ is the compression ratio. Details for this and alternative estimators are deferred to Appendix~\ref{app:method_details}.

\textbf{Step 3: Transferability of NSL Parameters.}
Only $\rho(T)$ is re-estimated per domain; the parameters $\nu, \mu$ are properties of the transformation type, not the dataset. Specifically, $\mu = 1$ for any deterministic transformation, since deterministic information loss induces a linear relationship between lost mutual information and Bayes-risk elevation (Lemma~\ref{lemma:ceiling}); $\nu$ is determined by how transformation-induced distortion propagates through parameter estimation, grounded in the Cramér-Rao bound; and from the proof of Lemma~\ref{lemma:ceiling}, the loss shift takes the form $\Phi\bigl(I(X;Y)(1-\rho)\bigr)$, so $\kappa$ implicitly absorbs the source-domain mutual information $I(X;Y)$. For (i), we reuse $(\hat\nu, \hat\mu, \hat\kappa)$ across noise levels. For (ii), transferring $\kappa$ from OpenWebText is reasonable as both corpora share an autoregressive prediction structure over natural language, making their context--target mutual information comparable. We fit $(\hat\nu, \hat\mu, \hat\kappa)$ once on the source-domain $\rho$-sweep, a leave-one-out fold across noise levels on full MIMIC-IV vital signs for (i), and the OpenWebText quantization sweep for (ii), and reuse them on the target. Besides, the model-capacity exponent $\alpha$ is invariant to data transformations (Prop.~\ref{prop:model_capacity}) and is inherited from the source; only $\beta$ and $E$ require prediction. The resulting values $\hat{\beta}_t$ and $\hat{E}_t$ obtained in Step~4 already account for the effects of variance inflation and optimal-loss shift.

\begin{figure}[t]
    \begin{minipage}{\textwidth}
    \centering
    \begin{tabular}{@{\hspace{-3.8ex}} c @{\hspace{-2.4ex}} c @{\hspace{-1.5ex}} c @{\hspace{-1.5ex}}}
        \begin{tabular}{c}
        \includegraphics[width=.31\textwidth]{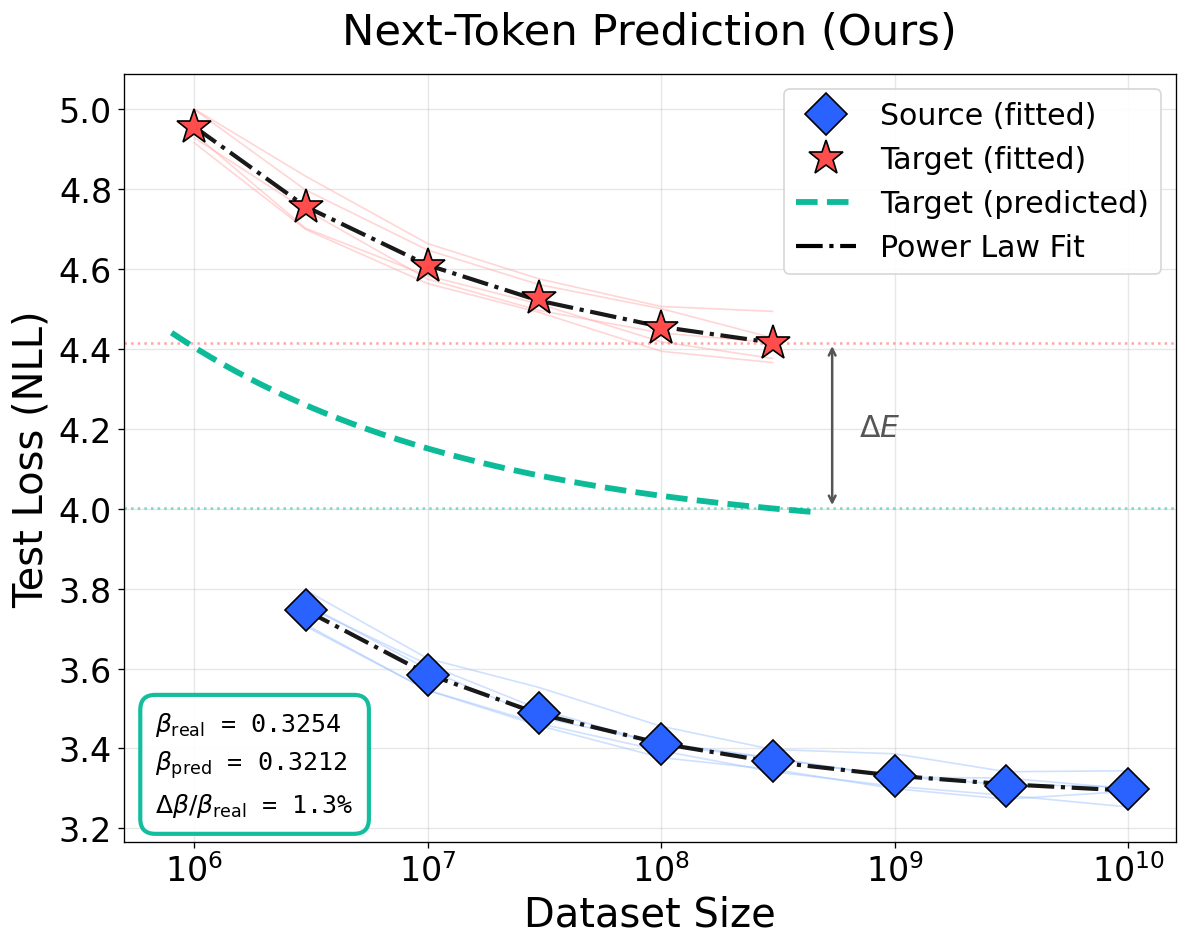}
        \\
        {\small{(a)}}
        \end{tabular} & 
        \begin{tabular}{c}
        \includegraphics[width=.31\textwidth]{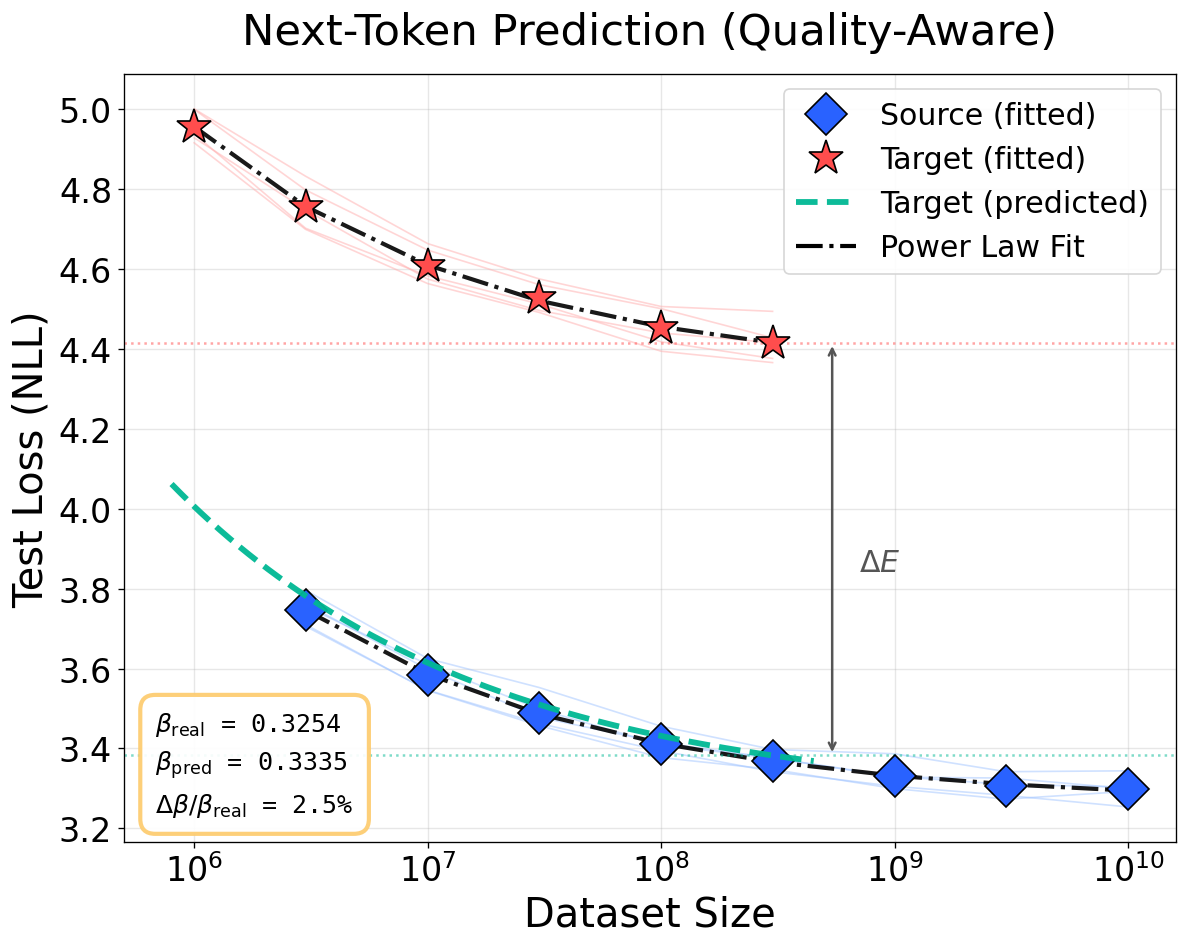} 
        \\
        {\small{(b)}}
        \end{tabular} &
        \begin{tabular}{c}
        \includegraphics[width=.31\textwidth]{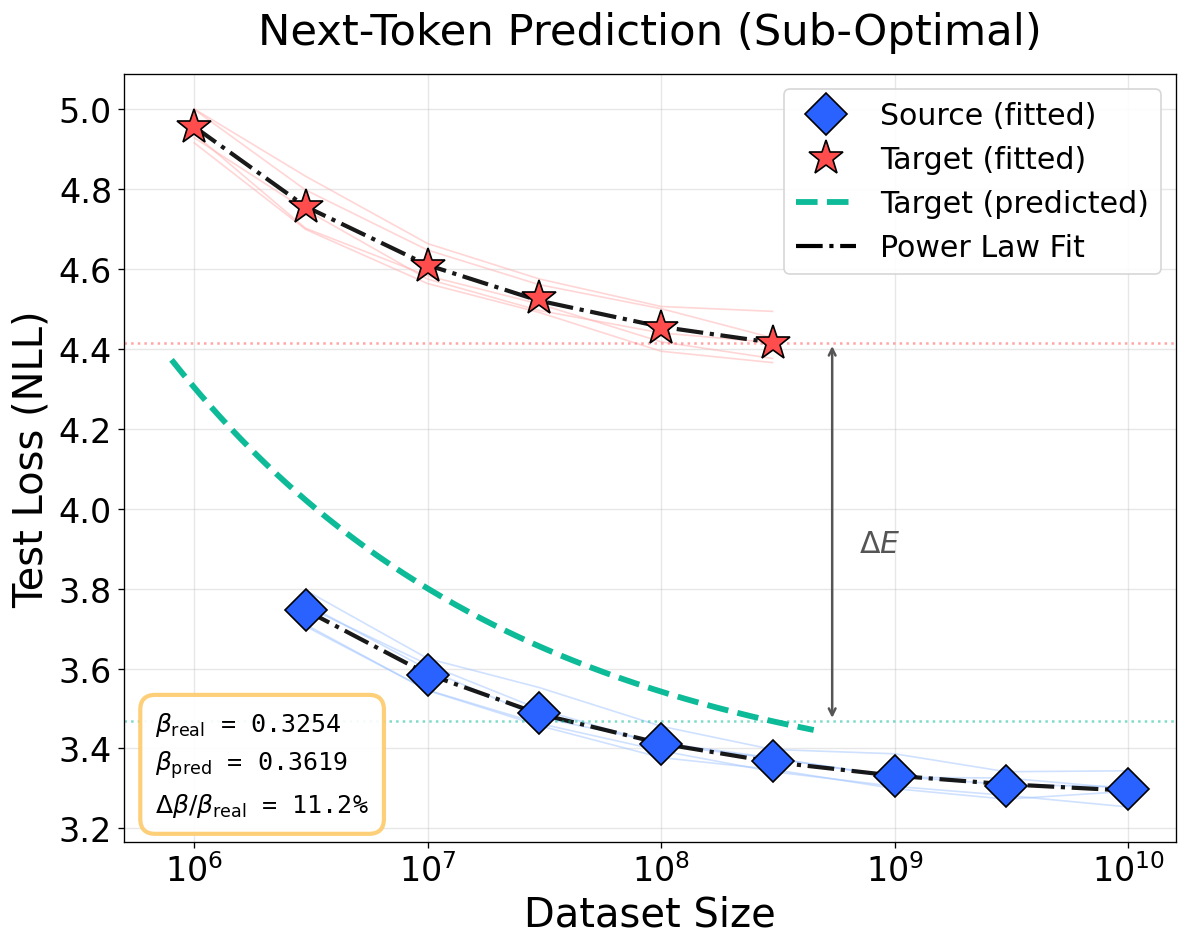} 
        \\
        {\small{(c)}}
        \end{tabular} \\
        \begin{tabular}{c}
        \includegraphics[width=.31\textwidth]{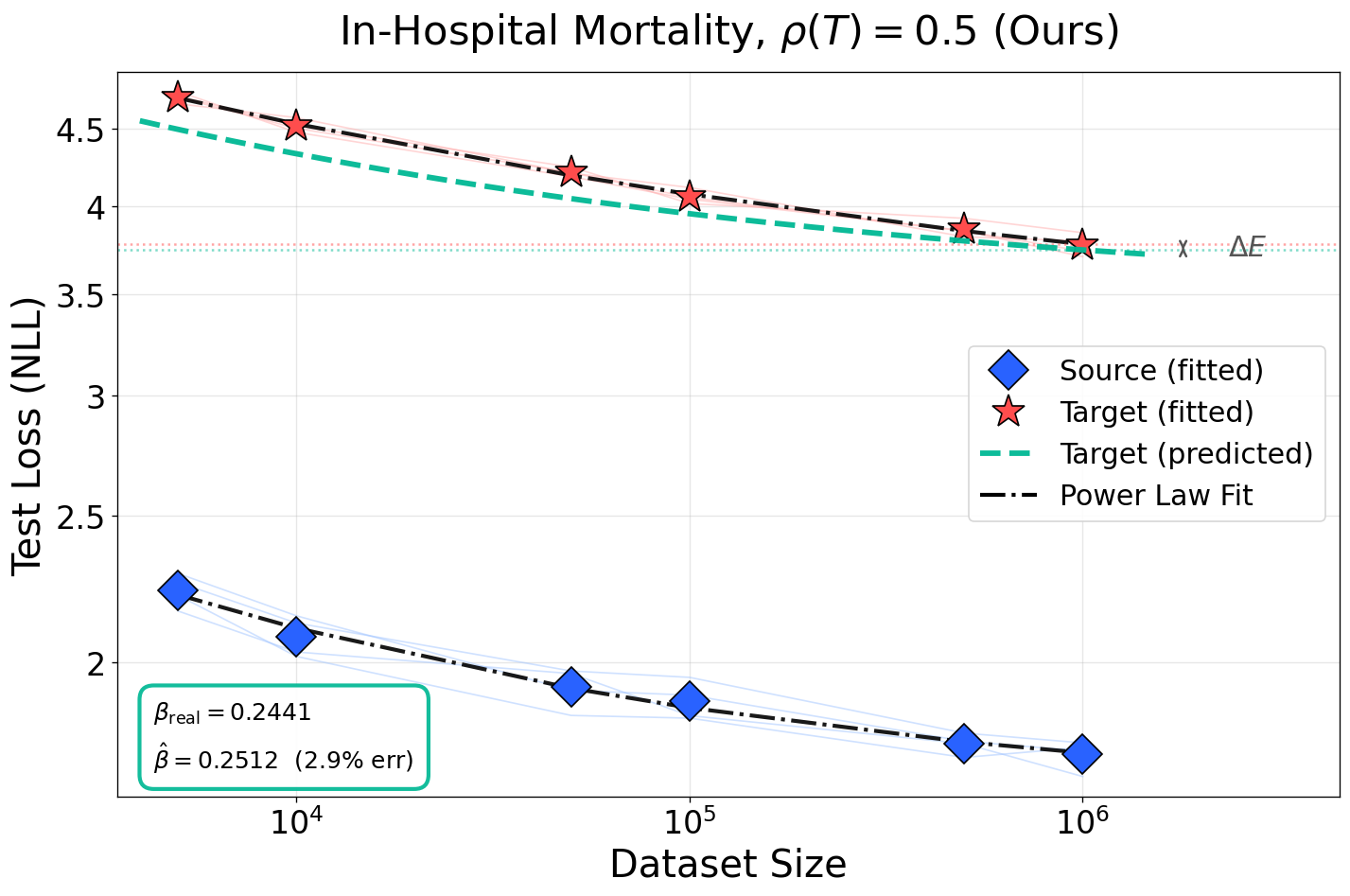}
        \\
        {\small{(d)}}
        \end{tabular} & 
        \begin{tabular}{c}
        \includegraphics[width=.31\textwidth]{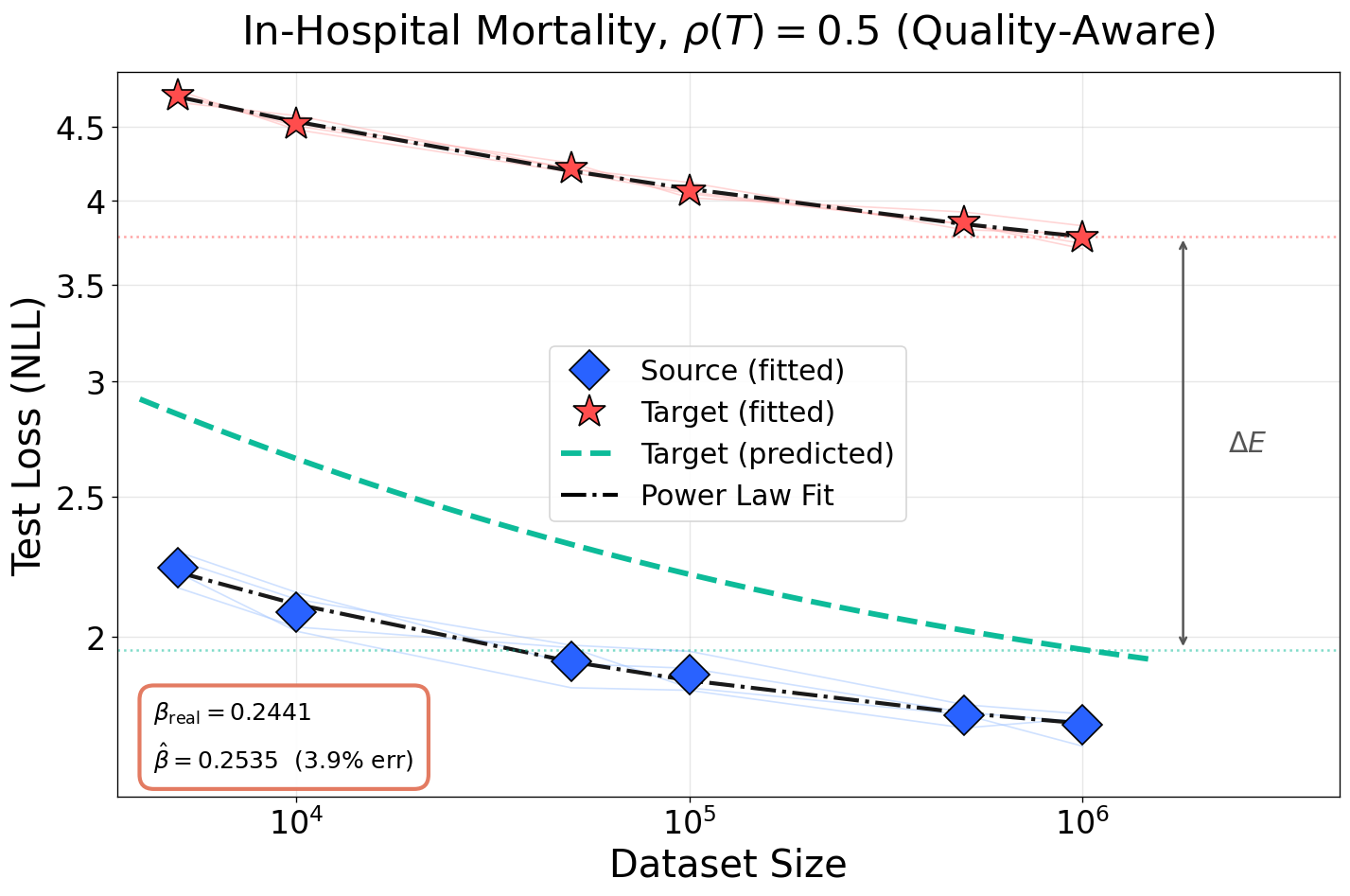} 
        \\
        {\small{(e)}}
        \end{tabular} &
        \begin{tabular}{c}
        \includegraphics[width=.31\textwidth]{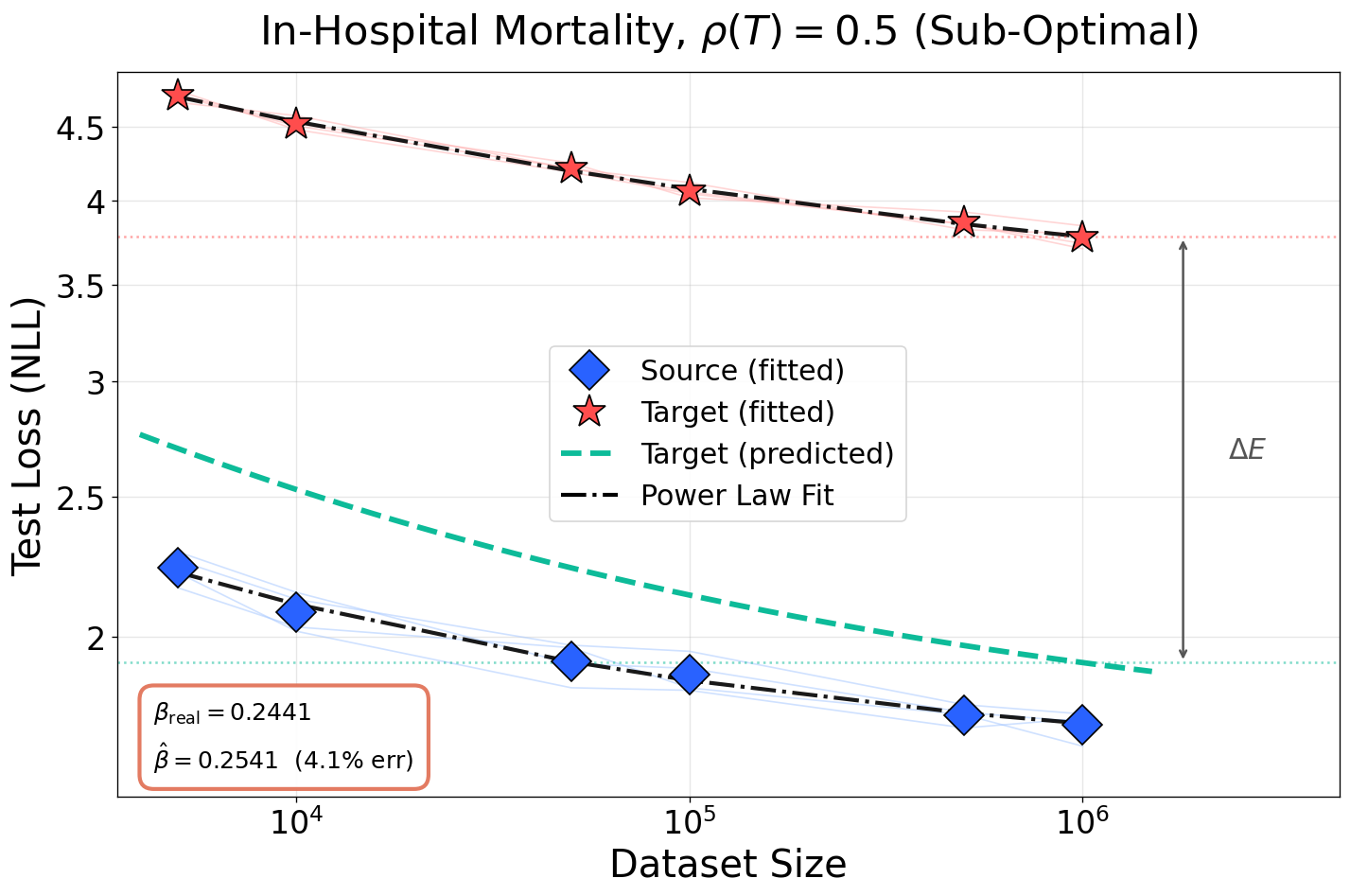} 
        \\
        {\small{(f)}}
        \end{tabular} \\
        \end{tabular}
    \end{minipage}
    \caption{Visualization of cross-domain scaling predictions compared with other baselines. We provide an accurate estimate (within 3\%) of both the data scaling exponent $\beta$ and the irreducible loss $E$ induced by information loss. The Quality-Aware~\citep{subramanyam2025scaling} scaling law partially captures the impact of data quality on the scaling exponent, but fails to account for the loss gap, which results in a larger discrepancy between the fitted and predicted scaling law curves. In contrast, the Sub-Optimal~\citep{chen2025revisiting} baseline shows higher estimation error in both $\beta$ and $E$. Our framework adapts to the loss shift as changing the corpus domain or the noise injection level. In contrast, the baselines rely on a constant-$E$ assumption, making them incapable of capturing such variations.}
    \label{fig:ntp_plots}
\end{figure}

\textbf{Results.}
Table~\ref{tab:nsl_pred} and Figure~\ref{fig:ntp_plots} reports predicted versus fitted $\beta$ and $E$ for both settings, alongside the Quality-Aware~\citep{subramanyam2025scaling} and Sub-Optimal~\citep{chen2025revisiting} baselines. Our framework recovers the data-scaling exponent $\beta$ to within $3\%$ relative error, substantially outperforming both baselines. In addition, neither of baselines model a quality-dependent loss floor and accordingly predict an unchanged $E$.

\subsection{Limitations of Current Framework}
\label{sec:limitations}
Our meta-scaling framework is most reliable when the source-to-target shift can be meaningfully cast as a reduction in information resolution, and its assumptions weaken as one moves outside this regime. We highlight three main limitations. \textbf{(i) Cross-domain transferability of $\nu, \mu, \kappa$.} The framework treats these as transformation-specific properties: $\nu$ from how distortion propagates through estimator variance (Cramér-Rao), $\mu = 1$ for any deterministic transformation, and $\kappa$ as a loss-shift coefficient that implicitly absorbs the source-domain mutual information $I(X;Y)$. Transferring $\kappa$ across corpora therefore assumes comparable task-relevant mutual information between source and target: reasonable for OpenWebText$\to$MIMIC-IV (both natural language, both autoregressive prediction) but increasingly tenuous for unrelated subdomains (e.g., natural text $\to$ source code), where shared information is minimal. \textbf{(ii) Estimating $\rho(T)$.} When the transformation is known by construction (noise injection, explicit quantization), $\rho(T)$ admits a closed form. For implicit cross-corpus shifts, $\rho(T)$ must be estimated from data, and our compression- and $n$-gram-based estimators only approximate the true $\rho_{\text{true}}$: they may fail when the target introduces new information absent from the source, or when information loss is non-uniform and concentrated on task-critical features. \textbf{(iii) Cross-modality scaling.} Our framework currently assumes a single shared modality between source and target; cross-modality prediction (e.g., text$\to$image) lacks a natural transformation $T$ within the present formulation, and the underlying loss functions are likely incomparable. A promising extension is to define $\rho(T)$ within a shared embedding space such as CLIP, but this remains future work.

\section{Conclusion and Future Work}\label{sec:conclusion}
Neural scaling laws are most valuable precisely where they are hardest to fit: in new domains, on new tasks, under data and compute constraints that preclude exhaustive sweeps. This paper addresses that gap through meta-scaling laws that characterize when scaling remains invariant and predict how it changes under data transformations. We show that scaling laws are preserved under bijective transformations of the training data and modified under non-bijective ones in information-theoretically grounded ways. Such modifications are through two distinct mechanisms, variance inflation and Bayes-risk elevation, that together yield the \textit{Information-Resolution Scaling Law}. Within its applicability regime, the framework predicts target-domain scaling from a source calibration and an estimate of $\rho(T)$. Empirically, this enables a law fit on general-purpose text or clean physiological signals to predict scaling behavior on clinical notes and noisy mortality data, recovering data-scaling exponents to within $3\%$ and the $\rho$-dependent loss floor that prior quality-aware formulations cannot capture.
Three directions appear most promising: extending $\rho$ estimation to compositional transformations whose effects do not factor cleanly; deriving cross-modality $\rho$ within shared embedding spaces (e.g., CLIP) to bridge the current single-modality boundary; and using the law's compute-optimal predictions $N^*(C, \rho)$ to guide curation policies that explicitly trade data quality for model size.\looseness=-1



\bibliography{references}
\bibliographystyle{abbrv}

\newpage
\appendix
\onecolumn

\vspace{1cm}
\centering{\textbf{\Large{Supplementary Material for 
``Meta-Scaling Laws: On the Invariance and Generality of Neural Scaling''}}}

\justifying
\setlength{\parindent}{0pt}
\vspace{0.5cm}

\section{Extended Related Works}
\textbf{Scaling Laws Across Modalities.}~
Neural scaling laws have become foundational to understanding and predicting model behavior, originating from systematic characterizations in language modeling \citep{kaplan2020scaling, hoffmann2022training} and subsequently extending across diverse data modalities. In computer vision, analogous power-law relationships have been established for Vision Transformers, with further refinements showing that compute-optimal configurations require joint optimization of width and depth rather than scaling either dimension in isolation \citep{zhai2022scaling, alabdulmohsin2023getting}. The multimodal setting has received particularly extensive investigation: autoregressive generative modeling across images, video, and image-text data exhibits nearly universal exponents relating optimal model size to compute \citep{henighan2020scaling}, while contrastive language-image pre-training follows reproducible power-law scaling for zero-shot classification and retrieval \citep{cherti2023reproducible}. Mixed-modal scaling laws have been derived across seven modalities, explicitly modeling both synergy and competition between modalities as additive terms to unimodal laws \citep{aghajanyan2023scaling}, with recent work on native multimodal models revealing that early-fusion architectures outperform late-fusion approaches at lower parameter counts \citep{shukor2025scaling}. Beyond vision and language, decoder-only transformers for time series forecasting exhibit power-law scaling analogous to language models across diverse temporal domains \citep{edwards2024scaling}. In this work, we demonstrate the application of our proposed framework across multiple modalities, highlighting its generalizability to diverse domains.\looseness=-1

\textbf{Predicting Scaling Laws.}~
A growing body of research has developed methods to predict scaling behavior without exhaustive training. Observational approaches demonstrate that model performance can be characterized through low-dimensional capability spaces, bypassing training entirely \citep{ruan2024observational}. Complementary strategies leverage existing models and training trajectories: fitting to intermediate checkpoints improves prediction accuracy, and scaling parameters often transfer across model families \citep{choshen2024hitchhiker}. Transfer-based prediction has been extended through loss-to-loss strategies using shifted power laws, enabling prediction across pretraining datasets and to downstream tasks \citep{brandfonbrener2024loss}. Beyond model-centric approaches, recent work incorporates data composition, predicting loss as a joint function of model size, training tokens, and domain weights to optimize data mixtures from small-scale runs \citep{shukor2025scaling}.  However, these methods treat data as a monolithic quantity characterized by volume or coarse domain labels, without principled account of why scaling exponents change across settings, or when they should remain invariant. 

\textbf{Data-Dependent Scaling Laws.}~
Recent work has extended classical scaling laws to account for data quality and composition, revealing several key phenomena. Early investigations into noise and architecture effects on data scaling found that quality degradation minimally impacts scaling exponents, suggesting compensation through increased data volume \citep{bansal2022data}. However, this view has been refined by studies demonstrating that data utility is heterogeneous: quality subsets exhibit differing scaling behavior, data repetition yields diminishing returns, and optimal filtering strategies must be tailored to available compute budgets \citep{goyal2024scaling}. Complementary work has leveraged scaling laws as diagnostic tools for data quality itself, using perplexity differences between models of varying sizes to reveal that scaling exponents vary predictably with data characteristics \citep{li2024scalingfilter}. More granular analyses have incorporated effective training tokens based on text diversity and syntheticity \citep{chang2024scaling}, while large-scale empirical studies across hundreds of models have identified high data density and suboptimal resource allocation as primary drivers of sub-scaling phenomena \citep{chen2025revisiting}. Despite these advances, existing formulations lack a unified theoretical framework explaining why different data properties modulate scaling behavior.

\textbf{Data Transformations and Applications}~
Transforming data into different representations has long been a fundamental operation in machine learning. Invertible transformations, which enable exact reconstruction—including linear mappings and bijective nonlinear functions—have been widely exploited in diverse settings such as normalizing flows \citep{rezende2015variational, papamakarios2021normalizing} and independent component analysis \citep{hyvarinen2023nonlinear}. In contrast, non-invertible transformations inherently discard information, as in quantization, clustering, and dimensionality reduction. Such methods are commonly employed to improve efficiency and reduce memory consumption, and have been extensively used in language modeling \citep{hyvarinen2023nonlinear, lang2024comprehensive}, image processing \citep{cosman2002using, lu2019learning}, and time-series forecasting \citep{alqahtani2021deep}. In the following sections, we categorize canonical transformations and analyze how they affect the components of scaling laws.

\section{Proofs of Propositions and Lemma}
\label{app:invariant_proofs}

\begin{proof}[\textbf{Proof of Proposition~\ref{prop:bayes_risk}}]
For any predictor $f$ on original data $X$, define $g = f \circ T^{-1}$ on transformed data $T(X)$. Then:
\begin{equation}
\mathbb{E}[\ell(g(T(X)), Y)] = \mathbb{E}[\ell(f(T^{-1}(T(X))), Y)] = \mathbb{E}[\ell(f(X), Y)]
\end{equation}
Since $T$ is invertible, this mapping $f \mapsto f \circ T^{-1}$ is a bijection between hypothesis classes that preserves risk. Taking infimum over all predictors:
\begin{equation}
R^*(T(X)) = \inf_g \mathbb{E}[\ell(g(T(X)), Y)] = \inf_f \mathbb{E}[\ell(f(X), Y)] = R^*(X)
\end{equation}
Thus the irreducible loss $E = R^*$ is invariant.
\end{proof}

\begin{proof}[\textbf{Proof of Proposition~\ref{prop:sample_complexity}}]
Sample complexity lower bounds depend on mutual information $I(X; Y)$ \citep{xu2017information}:
\begin{equation}
n \geq \Omega\left(\frac{I(X; Y)}{\epsilon^2}\right)
\end{equation}
By Lemma~\ref{thm:mi_invariance}, $I(T(X); Y) = I(X; Y)$ for invertible $T$, so the bound is identical.

For the upper bound, any learning algorithm $A$ on original data can be simulated on transformed data: given samples $\{(T(x_i), y_i)\}$, compute $\{(T^{-1}(T(x_i)), y_i)\} = \{(x_i, y_i)\}$ and run $A$. This simulation has identical sample complexity, so:
\begin{equation}
n^*(\epsilon; T(X)) = n^*(\epsilon; X)
\end{equation}
Thus the data scaling exponent $\beta$ is invariant.
\end{proof}


\begin{proof}[\textbf{Proof of Proposition~\ref{prop:model_capacity}}]
The Bayes optimal predictor for transformed data is $g^* = f^* \circ T^{-1}$.

\textit{Linear case} ($T(x) = Wx$): For a neural network approximating $f^*$, the composition $f^* \circ W^{-1}$ can be realized by replacing the first layer weights $W_1 \to W_1 W^{-1}$. This requires identical parameter count.

\textit{Nonlinear case} ($T$ a smooth diffeomorphism): If $f^* \in C^s$ (s-times differentiable), then $g^* = f^* \circ T^{-1} \in C^s$ since composition preserves smoothness. Neural network approximation rates depend only on smoothness \citep{yarotsky2017error}:
\begin{equation}
\inf \|f - \mathcal{N}\|_\infty = \Theta(N^{-s/d})
\end{equation}
Since $f^*$ and $g^*$ have identical smoothness, they require identical model capacity. Thus the parameter scaling exponent $\alpha$ is invariant.
\end{proof}

\begin{proof}[\textbf{Proof of Lemma~\ref{lemma:ceiling}}]\label{prf:lemma2}
We prove the result in two steps: deriving the gap for log-loss, then establishing irreversibility.

\textit{Step 1: Bayes-optimal loss gap.}
For log-loss, the Bayes-optimal loss given observation $Z$ equals the conditional entropy: $L^*_Z = H(Y|Z)$. Using $H(Y|Z) = H(Y) - I(Z;Y)$:
\begin{align}
L^*_{X'} - L^*_X &= H(Y|X') - H(Y|X) \nonumber \\
&= [H(Y) - I(X';Y)] - [H(Y) - I(X;Y)] \nonumber \\
&= I(X;Y) - I(X';Y) = I(X;Y)(1 - \rho(T))
\end{align}
For general losses, Fano's inequality \citep{scarlett2019introductory} implies $L^*_{X'} - L^*_X \geq \Phi(I(X;Y) - I(X';Y))$ where $\Phi$ is monotonically increasing with $\Phi(0) = 0$.

\textit{Step 2: Irreversibility via Data Processing Inequality.}
The DPI guarantees $I(X';Y) \leq I(X;Y)$ for any transformation, with strict inequality when $T$ is non-invertible. Crucially, no post-processing can recover lost information: for any function $g$, $I(g(X');Y) \leq I(X';Y)$. Therefore, no learning algorithm on $X'$ can achieve loss below $L^*_{X'}$, and the gap $\Phi(I(X;Y)(1-\rho))$ persists regardless of dataset size $D$ or model capacity $N$.

This distinguishes the ceiling from statistical efficiency loss: efficiency loss vanishes as $D \to \infty$, while the ceiling is a hard information-theoretic limit that no amount of data can overcome.
\end{proof}

\begin{proof}[\textbf{Proof of Proposition \ref{prop:info_capacity}}]\label{prf:prop2}
The result follows from decomposing the expected loss into three components.

\textit{(1) Bayes-optimal loss shift.}
By proof of Lemma~\ref{lemma:ceiling}, the Bayes-optimal loss on transformed data is:
\begin{equation}
L^*_{X'} = L^*_X + \Phi(I(X;Y)(1-\rho)) = E + \kappa(1-\rho)^\mu
\end{equation}
where $E = L^*_X$ is the clean-data Bayes risk and $\Phi(\cdot) = \kappa(\cdot)^\mu$ for appropriate $\kappa, \mu$.

\textit{(2) Statistical efficiency degradation.}
Information loss inflates estimator variance. By the Cramér-Rao bound \citep{van2004detection}, variance scales inversely with Fisher information, which is proportional to $I(X';Y) = \rho \cdot I(X;Y)$. For $D$ samples:
\begin{equation}
\text{Var}(\hat{f}) \propto \frac{1}{D \cdot \rho^\nu}
\end{equation}
where $\nu$ depends on corruption type ($\nu = 1$ for additive noise, $\nu = 2$ for label noise).

\textit{(3) Combining terms.}
Standard bias-variance decomposition gives $\mathbb{E}[L] = L^* + \text{estimation error} + \text{approximation error}$. With sufficient model capacity, approximation error scales as $A/N^\alpha$. Estimation error scales as $(B/D^\beta) \cdot \rho^{-\nu}$. The Bayes risk is $L^*_{X'} = E + \kappa(1-\rho)^\mu$. Combining:
\begin{equation}
L(N, D, \rho) = \frac{A}{N^\alpha} + \frac{B}{D^\beta} \cdot \rho^{-\nu} + E + \kappa(1-\rho)^\mu
\end{equation}

The boundary conditions follow directly: $\rho = 1$ gives $\rho^{-\nu} = 1$ and $(1-\rho)^\mu = 0$; $\rho \to 0$ gives $\rho^{-\nu} \to \infty$ (but dominated by ceiling for large $D$) and $(1-\rho)^\mu \to 1$.
\end{proof}

\section{Additional Methodological Details}\label{app:method_details}
\subsection{Information Resolution for Specific Transformations}\label{subsec:info_resolution}

We discuss practical usages of our proposed framework and its potential applications for predicting scaling laws in different domains. We first derive the corresponding scaling parameters for two common non-bijective transformations. More detailed discussions can be found in Appendix. 

\textbf{Quantization.}
Consider $q$-level quantization of a discrete source with vocabulary $\mathcal{V}$, $|\mathcal{V}| = V$. If the original tokens are approximately uniformly distributed, then $H(X) \approx \log V$. After quantization to $q$ centroids, the maximum entropy is $H(X') \leq \log q$. By the data processing inequality applied to the Markov chain $Y \to X \to X'$:
\begin{equation}
\rho_{\text{quant}} = \frac{I(X';Y)}{I(X;Y)} \leq \frac{H(X')}{H(X)} \approx \frac{\log q}{\log V}.
\label{eq:quant}
\end{equation}
The variance inflation term is $\phi_{\text{VI}} = \rho^{-\nu}$ with the exponent $\nu$ determined by how quantization error propagates through learning. The optimal loss shift arises because quantization irreversibly merges $V/q$ tokens per cluster on average, discarding fine-grained distinctions:
\begin{equation}
\phi_{\text{LS}} = \kappa \left(1 - \frac{\log q}{\log V}\right) = \kappa(1 - \rho).
\end{equation}

\paragraph{Low-Rank Projection.}
Let $X \in \mathbb{R}^d$ with covariance $\Sigma = \sum_{i=1}^d \lambda_i u_i u_i^\top$ (eigenvalues sorted descending). Projection onto the top-$k$ subspace, $X' = U_k U_k^\top X$, discards variance in the trailing $(d$-$k)$ directions. For jointly Gaussian $(X, Y)$, mutual information decomposes across independent components, yielding:
\begin{equation}
\rho_{\text{rank}} = \frac{I(X';Y)}{I(X;Y)} = \frac{\sum_{i=1}^k \lambda_i \cdot r_i^2}{\sum_{i=1}^d \lambda_i \cdot r_i^2}
\label{eq:low_rank}
\end{equation}
where $r_i^2$ is the squared correlation between the $i$-th principal component and $Y$. When $Y$ loads uniformly across components, this simplifies to the variance ratio $\sum_{i=1}^k \lambda_i / \sum_{i=1}^d \lambda_i$. The optimal loss shift is exactly the information in discarded dimensions:
\begin{equation}
\phi_{\text{LS}} = \kappa(1 - \rho) = \kappa \cdot \frac{\sum_{i=k+1}^d \lambda_i}{\sum_{i=1}^d \lambda_i}
\end{equation}
Since projection is a linear operation, the variance inflation exponent is $\nu = 1$.

\textbf{Additive Noise.}
For $X' = X + \epsilon$ with $\epsilon \sim \mathcal{N}(0, \sigma_\epsilon^2 I)$ independent of $X$, the mutual information under a joint Gaussian $(X, Y)$ satisfies:
\begin{equation}
I(X';Y) = \frac{1}{2}\log\left(1 + \frac{\sigma_X^2}{\sigma_\epsilon^2}\right) = \frac{1}{2}\log(1 + \text{SNR})
\label{eq:noise}
\end{equation}
Thus $\rho_{\text{noise}} = \log(1 + \text{SNR})/\log(1 + \text{SNR}_0)$ where $\text{SNR}_0$ is the baseline. In the high-SNR regime, $\rho \approx \text{SNR}/\text{SNR}_0$. From the Cramér-Rao bound, estimator variance scales inversely with Fisher information \citep{van2004detection}, which is proportional to SNR: this gives $\phi_{\text{VI}} = \rho^{-1}$. Unlike deterministic transformations, additive noise creates a soft ceiling: information is obscured rather than deleted, so $\phi_{\text{LS}}$ grows more gradually, scaling as $\kappa(1 - \rho)^\mu$ with $\mu < 1$ reflecting partial recovery through averaging.

\subsection{Information Resolution Estimators for Cross-domain Scaling Predictions.}\label{app:rho_estimator}

The naive vocabulary-size estimator $\hat{\rho}_{\text{vocab}}(T) = \log q / \log V$ from Eq.~\eqref{eq:quant} provides a useful starting point but rests on two unrealistic assumptions: (i) tokens are approximately uniformly distributed, and (ii) vocabulary size is the dominant axis of cross-corpus variation. Real corpora violate both: word frequencies are heavy-tailed, and shifts between domains involve syntax, local semantics, and repetition structure in addition to vocabulary. Because $\hat{\rho}_{\text{vocab}}$ provides an \emph{upper bound} on the true $\rho(T)$, it systematically underestimates the loss shift $\hat{E}_t = E_s + \hat{\kappa}(1-\hat{\rho})^{\hat{\mu}}$, propagating bias into downstream predictions. We therefore consider two refinements.

\paragraph{Estimator 1: Empirical $n$-gram entropy ratio.}
Rather than counting vocabulary, we estimate the entropy of the empirical token distribution directly:
\begin{equation}
\hat{\rho}_{\text{emp}}(T) = \frac{H(\hat{p}_{\mathcal{D}'})}{H(\hat{p}_{\mathcal{D}})},
\end{equation}
where $\hat{p}_{\mathcal{D}}$ and $\hat{p}_{\mathcal{D}'}$ are unigram distributions on the source and target corpora. To capture sequential dependencies, we extend this to the 3-gram conditional entropy:
\begin{equation}
\hat{\rho}_{\text{3-gram}}(T) = \frac{\hat{H}_3(\mathcal{D}')}{\hat{H}_3(\mathcal{D})}, \quad \hat{H}_3(\mathcal{D}) = -\!\!\sum_{w_1,w_2,w_3}\!\! \hat{p}(w_1,w_2,w_3)\log \hat{p}(w_3 \mid w_1 w_2).
\label{eq:rho_3gram}
\end{equation}

\textbf{Pros:} captures frequency skew, token diversity, and local syntactic structure simultaneously, so a clinical corpus with formulaic templates registers as lower-$\rho$ even when its raw vocabulary size matches another domain's. In practice, 3- or 4-gram statistics already account for most local syntactic effects.
\textbf{Cons:} requires computing joint distributions over $|\mathcal{V}|^n$ cells, which becomes memory- and sample-bound for large vocabularies; estimates degrade for rare $n$-grams and require smoothing; and it remains blind to long-range semantic structure beyond the $n$-gram window.

\paragraph{Estimator 2: Compression-based.}
Inspired by recent work on data quality in language models~\citep{chang2024scaling}, we use a general-purpose compressor (gzip) as a model-free probe of information content. The compression ratio is
\begin{equation}
\hat{\rho}_{\text{compress}}(T) = \frac{C(\mathcal{D}')/|\mathcal{D}'|}{C(\mathcal{D})/|\mathcal{D}|},
\label{eq:rho_compress}
\end{equation}
where $C(\cdot)$ denotes compressed size in bits and $|\cdot|$ denotes raw corpus size. A more compressible corpus is, by Kolmogorov-complexity intuition, lower in usable information per token and hence lower in $\rho$.

\textbf{Pros:} simultaneously reflects vocabulary, frequency, syntax, local semantics, and repetition structure without explicitly modeling any of them; computationally cheaper than constructing high-order $n$-gram tables; entirely non-parametric, so no smoothing or vocabulary-size hyperparameter is required.
\textbf{Cons:} the compressor's inductive bias is fixed and may not align with task-relevant information: gzip rewards local repetition but is insensitive to long-range semantic structure that a downstream language model would actually exploit; absolute compression ratios depend on corpus size and the compressor's dictionary, so only \emph{ratios} are interpretable.

\paragraph{Empirical comparison.}
Table~\ref{tab:rho_estimators} reports $\hat{\rho}(T)$ from all three estimators on the OpenWebText$\to$MIMIC-IV discharge-summary transfer, alongside the resulting predicted loss shift $\hat{E}_t$ and its gap to the empirical $E_t$. The vocabulary-based estimator yields $\hat{\rho}_{\text{vocab}} \approx 0.71$ and the largest $\hat{E}_t$ gap, consistent with its upper-bound bias: by inflating $\rho$, it shrinks $(1-\hat{\rho})^{\hat{\mu}}$ and understates the loss-floor elevation. The 3-gram estimator returns a lower $\hat{\rho}$: clinical text's templated structure makes its conditional entropy meaningfully smaller than OpenWebText's, and tightens the $\hat{E}_t$ gap accordingly. The compression-based estimator produces the lowest $\hat{\rho}$ and the smallest gap, reflecting its sensitivity to the formulaic repetition that dominates discharge summaries (e.g., section headers, recurring drug names and abbreviations), exactly the redundancy that $n$-gram statistics partially miss because such patterns span variable distances. Across estimators, the ordering of $\hat{E}_t$ accuracy tracks the ordering of $\hat{\rho}$ richness: the more dimensions of the cross-corpus shift the estimator captures, the closer the predicted loss floor sits to the empirical one. The remaining gap under $\hat{\rho}_{\text{compress}}$ is small enough to be consistent with our theoretical framework's assumption that $\nu, \mu, \kappa$ transfer cleanly between deterministic transformations on natural-language corpora.

\subsection{Robustness of the Information Resolution Scaling Law}
Holding the calibrated coefficients fixed, an error in $\hat{\rho}$ affects both the data term and the predicted loss floor. For $0<\hat{\rho}<1$, Eq.~\eqref{eq:info_capacity} gives
\begin{equation}
\frac{\partial \hat{E}_t}{\partial \hat{\rho}}
=-\kappa\mu(1-\hat{\rho})^{\mu-1},
\qquad
\frac{\partial\log\phi_{\mathrm{VI}}}{\partial\log\hat{\rho}}=-\nu,
\label{eq:rho_sensitivity}
\end{equation}
where $\hat{E}_t=E_s+\kappa(1-\hat{\rho})^\mu$ and $\phi_{\mathrm{VI}}=(B/D^\beta)\hat{\rho}^{-\nu}$. Thus, a small relative error $\epsilon=\Delta\hat{\rho}/\hat{\rho}$ induces a relative data-term error of approximately $-\nu\epsilon$. For $\mu=1$, the loss-floor error is exactly $\Delta\hat{E}_t=-\kappa\Delta\hat{\rho}$. Using the cross-corpus coefficients in Table~\ref{tab:rho_estimators} ($\nu=0.19$, $\kappa=2.61$, $\mu=1$, $\hat{\rho}=0.54$), a $10\%$ relative perturbation in $\hat{\rho}$ changes the data term by approximately $1.9\%$ and the predicted floor by $0.141$ in absolute loss. Increasing $D$ reduces the absolute error in the data term, but does not remove the floor error. Sensitivity can be larger near the boundaries: the data term grows rapidly as $\hat{\rho}\to0$, and for $\mu<1$ the floor derivative diverges as $\hat{\rho}\to1$.

These expressions quantify sensitivity of the predicted loss surface. At fixed $\rho$, the factor $\rho^{-\nu}$ rescales $B$ without changing $\beta$: $B D^{-\beta}\rho^{-\nu}=(B\rho^{-\nu})D^{-\beta}$. Sensitivity of an effective exponent obtained through a constrained finite-range fit therefore depends on the fitting protocol; it cannot be inferred from $\rho$ alone. The estimator comparison in Table~\ref{tab:rho_estimators} complements this local analysis by showing how different proxies for $\rho$ affect the predicted loss floor.

\subsection{Cost of Obtaining Coefficients of Information Resolution Scaling Law}
The framework replaces repeated target-domain scaling sweeps with a source calibration that can be reused under compatible transformations. Algorithm~\ref{alg:predict_scaling} requires a source sweep over model size, dataset size, and transformation strength to fit the base law and its information-resolution coefficients. Parameters describing the distortion mechanism can be reused when the transformation family and modeling assumptions remain applicable. A justified closed-form expression for $\rho$ further avoids estimating it from the target corpus, but does not eliminate the initial calibration cost.

The loss-shift coefficient $\kappa$ requires particular care. For the exact log-loss identity with task-specific mutual information, $\mu=1$ and $\kappa=I(X;Y)$, so $\kappa$ depends on the source distribution. Transferring it assumes comparable task-relevant information in the source and target; distant transfers may instead require target-side calibration across transformation strengths. The computational benefit is therefore greatest when several predictions share a valid source calibration. Residual loss-floor errors can reflect either an inaccurate $\hat{\rho}$ or a change in $\kappa$, and do not by themselves distinguish these explanations.\looseness=-1

\subsection{Additional Discussions}
\paragraph{Scope of scalar sufficiency.}
For a fixed source distribution and prediction target $Y$, information resolution exactly determines the Bayes loss-floor shift under log-loss:
\begin{equation}
L^*_{X'}-L^*_X
=H(Y\mid X')-H(Y\mid X)
=I(X;Y)(1-\rho(T)).
\label{eq:scalar_floor_scope}
\end{equation}
Two transformations with the same task-specific $\rho$ therefore induce the same Bayes loss shift in this setting. This identity holds for both deterministic and stochastic input transformations. It does not establish that $\rho$ alone determines the full finite-scale scaling law, which also depends on the calibrated coefficients and assumptions about approximation and estimation error. In particular, a fitted $\mu\ne1$ should be interpreted as a phenomenological extension when using an information proxy or modeling finite-scale behavior, rather than as a consequence of stochasticity alone.

\paragraph{Task dependence and composition.}
A transformation can preserve information relevant to one target while discarding features needed for another, so $\rho$ is task-relative. For a pipeline $X'=T_1(X)$, $X''=T_2(X')$ satisfying $Y\to X\to X'\to X''$, data processing implies $\rho(T_2\circ T_1)\leq\rho(T_1)$. However, resolutions measured separately on the original input distribution do not determine the composite resolution: the second transformation acts on the information retained by the first. Evaluating the second ratio on that intermediate distribution gives the exact product $\rho(T_2\circ T_1)=\rho(T_1)I(X'';Y)/I(X';Y)$ when $I(X';Y)>0$. A pipeline therefore requires information about the intermediate representation, beyond standalone transformation summaries.

\paragraph{Emergent behavior.}
Nonlinear or discontinuous evaluation metrics can produce apparent emergence from smoothly improving model predictions~\citep{schaeffer2023emergent}. Such changes need not contradict a smooth loss prediction. A genuine transition in the loss-scaling regime, however, invalidates the corresponding base-law approximation. Our framework inherits the validity range of the underlying Chinchilla-form law and does not independently account for such transitions.

\section{Additional Results} \label{app:additional_results}
In this section, we present additional experimental results supplementary to the main experiments. Figures~\ref{fig:ihm_plots_075} and \ref{fig:ihm_plots_025} illustrate the cross-corpus next-token prediction experiments when $\rho(T)$ is set to $0.75$ and $0.25$, respectively. Tables~\ref{tab:pixel_prediction} and \ref{tab:snr_prediction_real_075} present scaling law prediction results on the ImageNet-1K dataset under two transformation types: image quantization and noise injection. Tables~\ref{tab:rho_estimators} through \ref{tab:snr_prediction_025} report the complete information-resolution scaling law parameter results for the scaling law prediction experiments supplementary to Table~\ref{tab:nsl_pred}. 

\begin{figure}[t]
    \begin{minipage}{\textwidth}
    \centering
    \begin{tabular}{@{\hspace{-3.8ex}} c @{\hspace{-2.4ex}} c @{\hspace{-1.5ex}} c @{\hspace{-1.5ex}}}
        \begin{tabular}{c}
        \includegraphics[width=.31\textwidth]{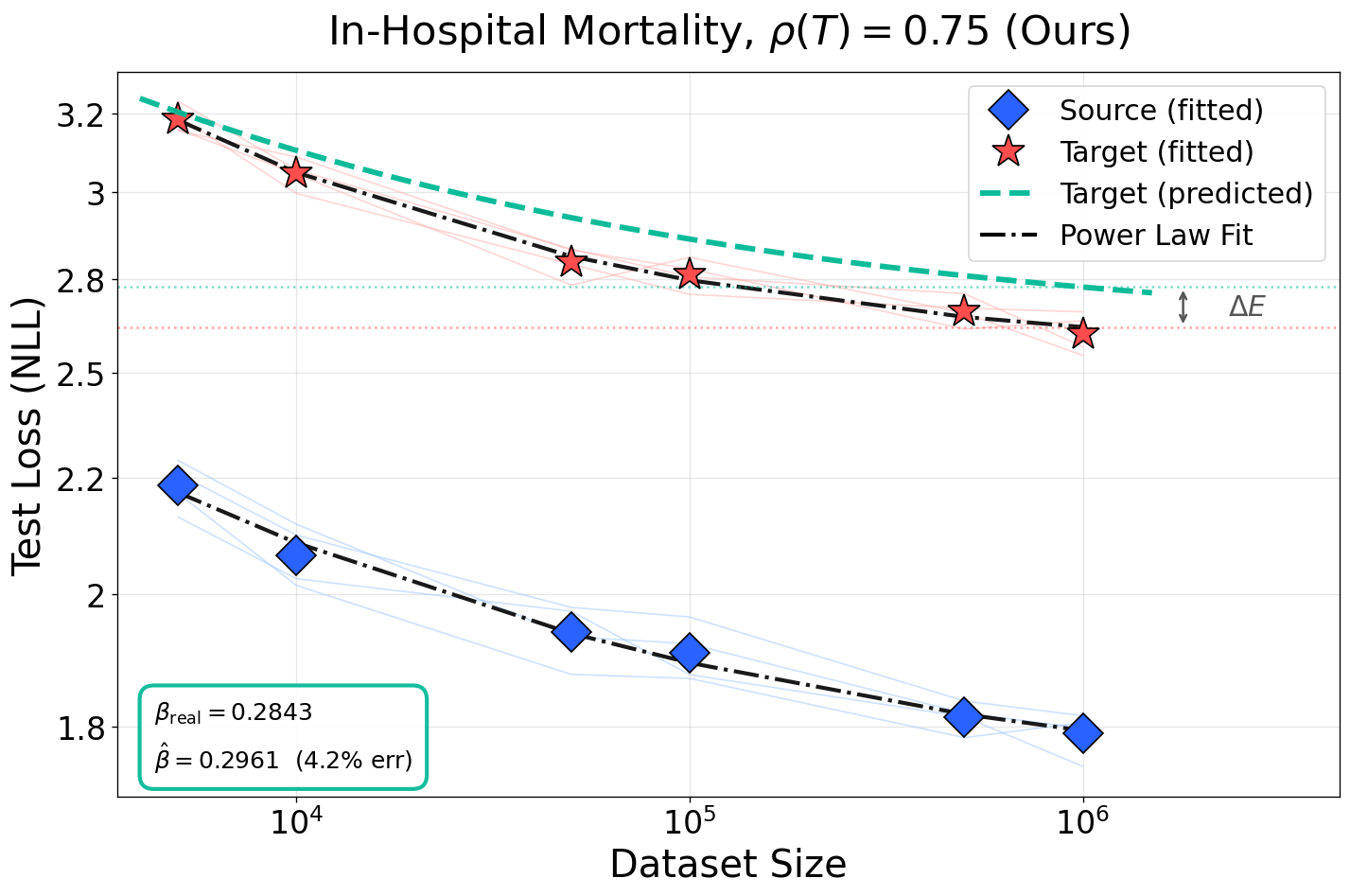}
        \\
        {\small{(a)}}
        \end{tabular} & 
        \begin{tabular}{c}
        \includegraphics[width=.31\textwidth]{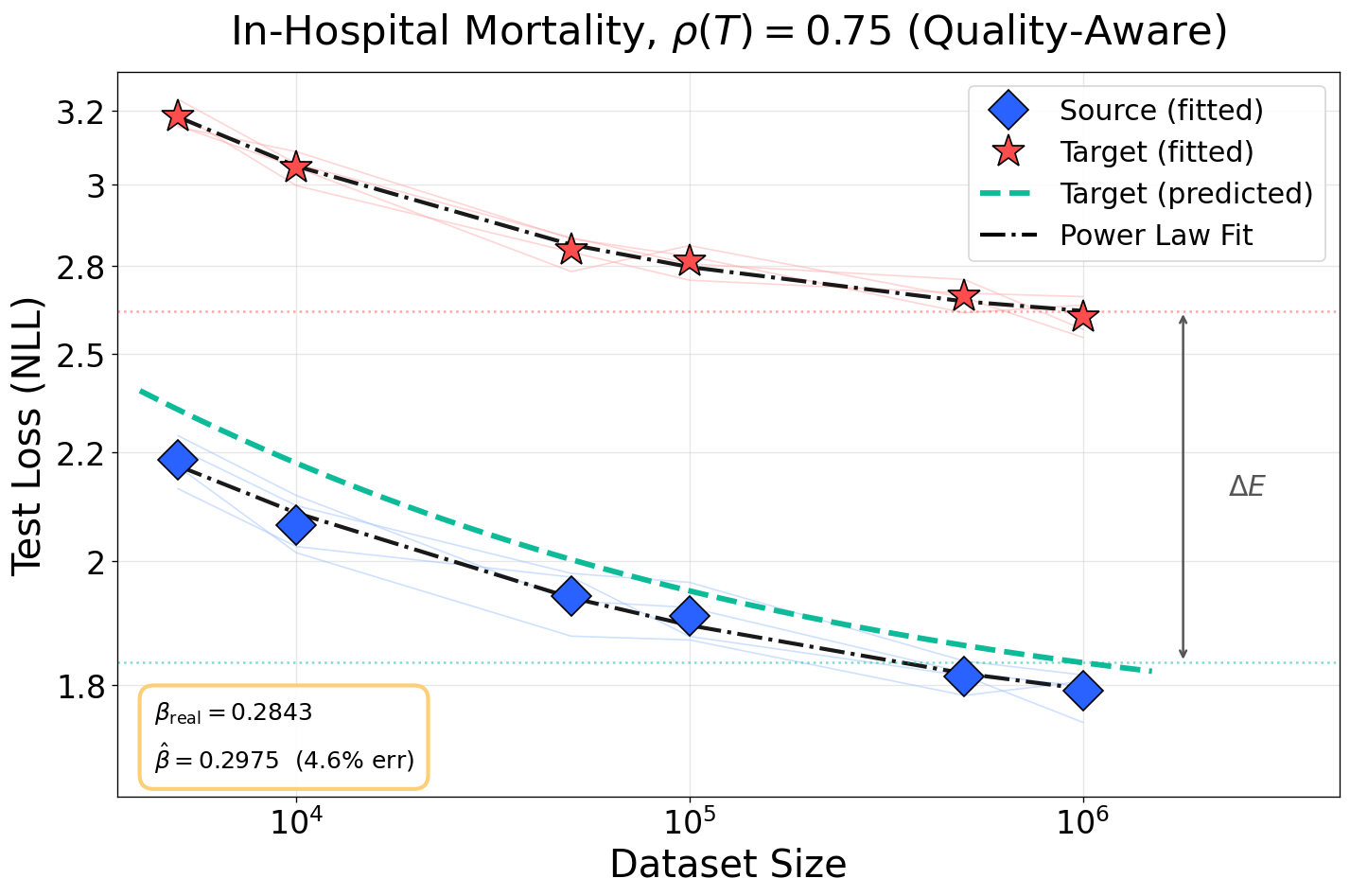} 
        \\
        {\small{(b)}}
        \end{tabular} &
        \begin{tabular}{c}
        \includegraphics[width=.31\textwidth]{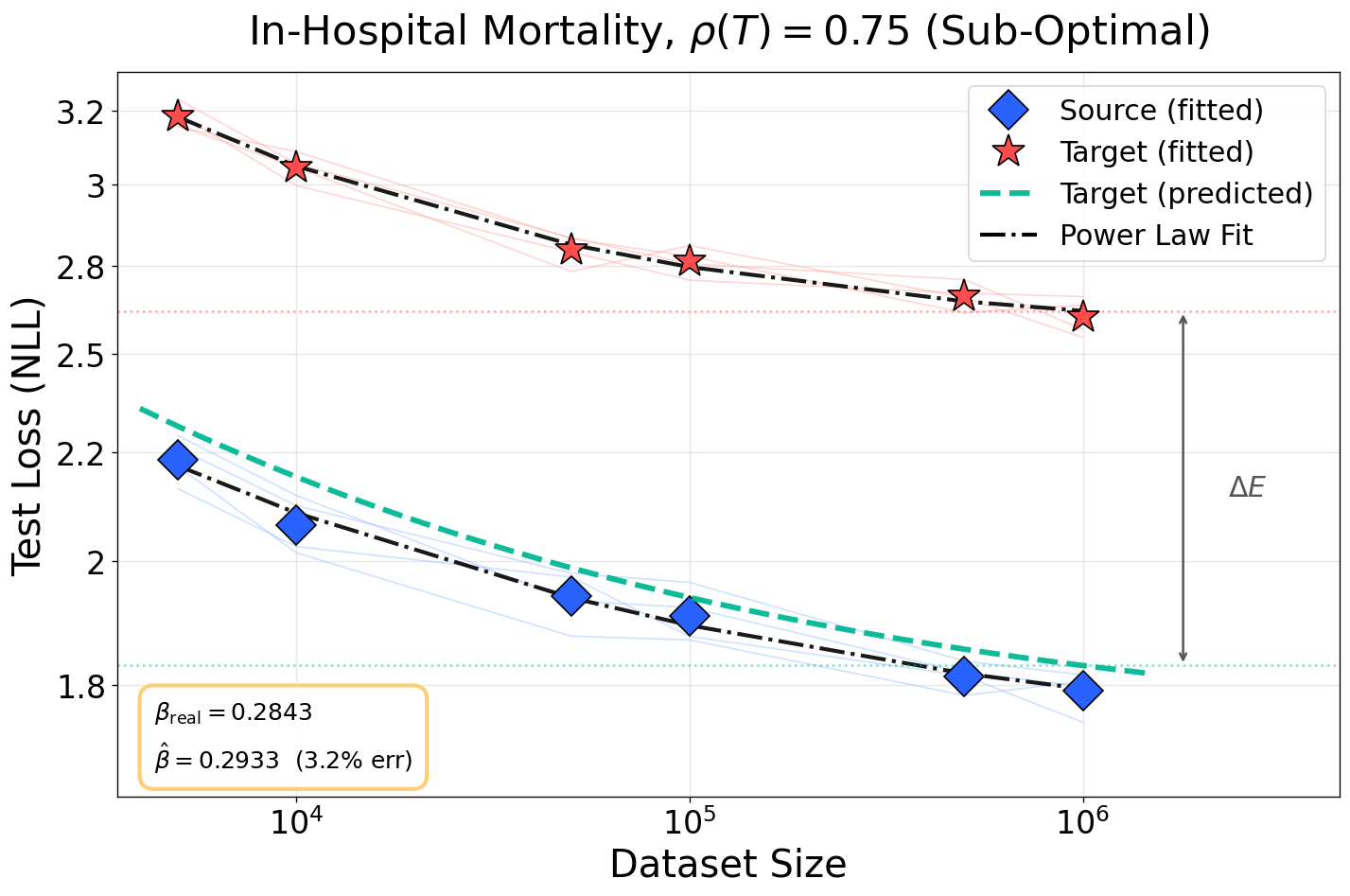} 
        \\
        {\small{(c)}}
        \end{tabular} \\
        \end{tabular}
    \end{minipage}
    \caption{We compare our method with baselines under different noise injection levels for the in-hospital mortality prediction task on MIMIC-IV. The reported results for our method follow the same experimental settings as in Table \ref{tab:nsl_pred} of the main paper. Aside from the gap induced by the irreducible loss, all three frameworks achieve similarly strong estimation performance on $\beta$.}
    \label{fig:ihm_plots_075}
\end{figure}

\begin{figure}[t]
    \begin{minipage}{\textwidth}
    \centering
    \begin{tabular}{@{\hspace{-3.8ex}} c @{\hspace{-2.4ex}} c @{\hspace{-1.5ex}} c @{\hspace{-1.5ex}}}
        \begin{tabular}{c}
        \includegraphics[width=.31\textwidth]{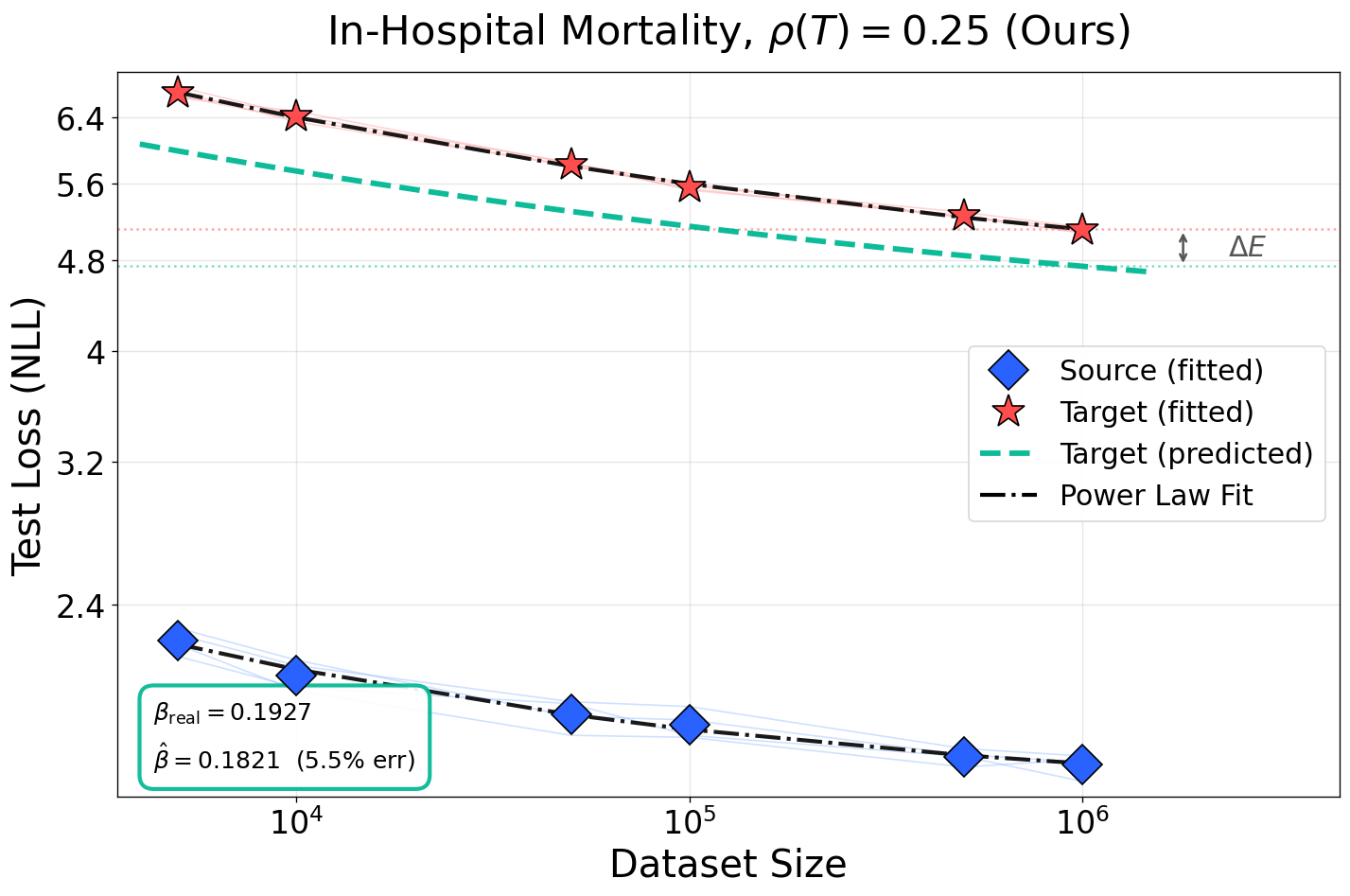}
        \\
        {\small{(a)}}
        \end{tabular} & 
        \begin{tabular}{c}
        \includegraphics[width=.31\textwidth]{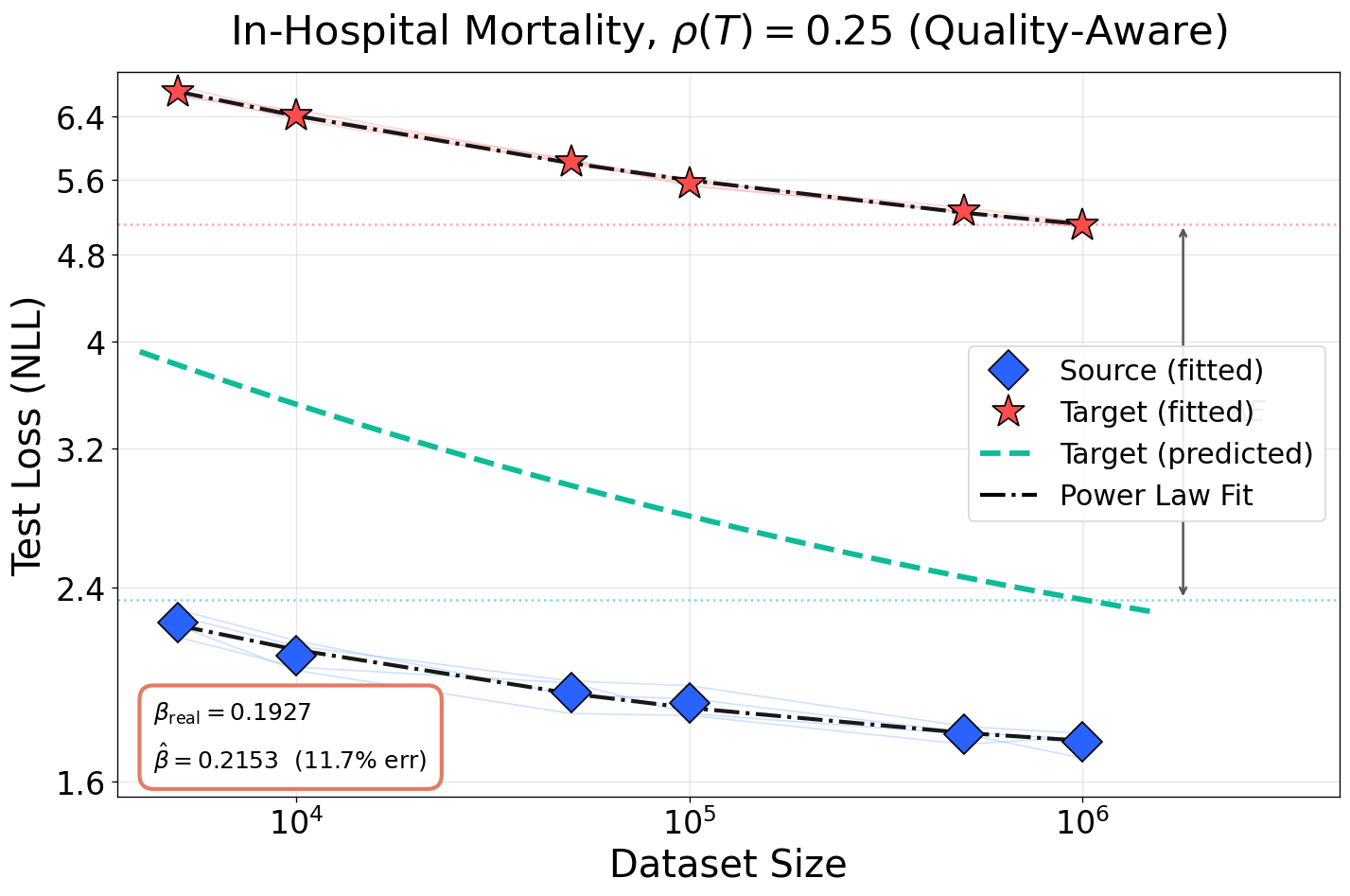} 
        \\
        {\small{(b)}}
        \end{tabular} &
        \begin{tabular}{c}
        \includegraphics[width=.31\textwidth]{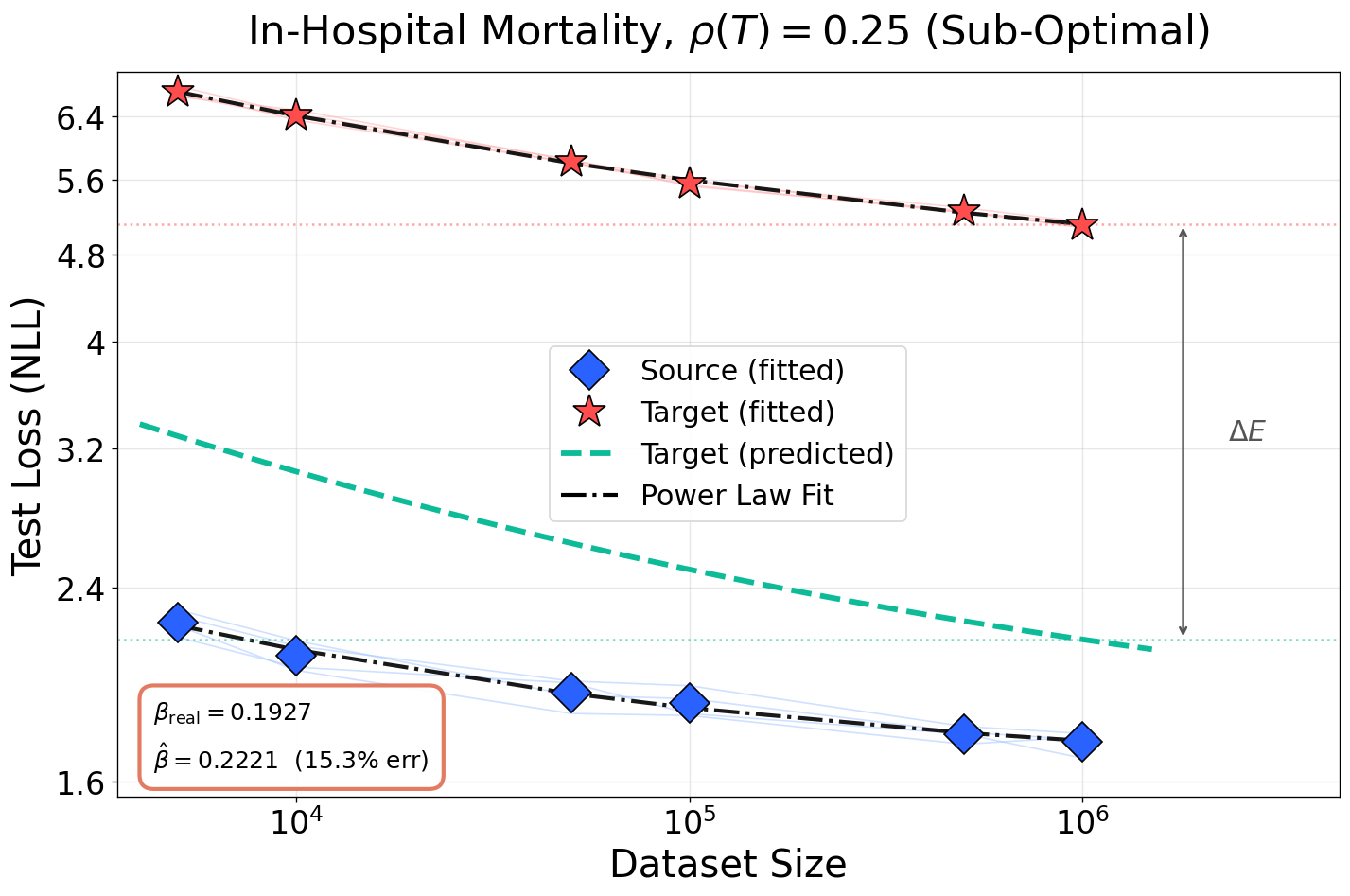} 
        \\
        {\small{(c)}}
        \end{tabular} \\
        \end{tabular}
    \end{minipage}
    \caption{At higher noise injection levels, our method still provides reasonable estimates (within an acceptable error range) for the key scaling law exponents. In contrast, for the baselines, not only does the irreducible loss persist, but the estimation of $\beta$ also shifts under these conditions.}
    \label{fig:ihm_plots_025}
\end{figure}

\begin{table}[h]
\centering
\caption{Cross-domain prediction results on the ImageNet-1K dataset. We evaluate scaling law parameter estimation under image quantization transformations and report all relevant parameters in the corresponding results.}
\label{tab:pixel_prediction}
\begin{tabular}{lcccc}
\toprule
\rowcolor{gray!25}
\textbf{Parameter} & \textbf{Formula} & \textbf{Values} & \textbf{Ground Truth} \\
\midrule
$\rho$               & Pixel number to point          & 0.25   & ---   \\
$\rho^{-\nu}$        & ---                            & 1.233  & ---   \\
$B_{\text{eff}}$     & $B_s \cdot \rho^{-\nu}$        & 43.16  & ---   \\
$\kappa(1-\rho)^\mu$ & ---                            & 2.025  & ---   \\
$\hat{E}_t$          & $E_s + \kappa(1-\rho)^\mu$     & 4.315  & 4.46  \\
$\hat{\beta}_t$      & refitted                       & 0.225  & 0.197 \\
\midrule
\rowcolor{gray!25}
\multicolumn{4}{c}{\textbf{Fixed parameters (from source domain fitting)}} \\
\multicolumn{4}{c}{$E_s = 2.29$ (irreducible loss)} \\
\multicolumn{4}{c}{$\beta_s = 0.28$ (source data-scaling exponent)} \\
\multicolumn{4}{c}{$\nu = 0.15$, $\kappa = 2.70$, $\mu = 1.0$ (transformation-type parameters)} \\
\multicolumn{4}{c}{$\alpha_s = 0.31$, $A_s = 20.03$, $B_s = 34.87$} \\
\bottomrule
\end{tabular}
\end{table}

\begin{table}[h]
\centering
\caption{Cross-domain prediction results on the ImageNet-1K dataset. We evaluate scaling law parameter estimation under noise injection and report all relevant parameters in the corresponding results. These experiments demonstrate that Info-Resolution Scaling Law provides reasonably accurate predictions for the key scaling law exponents under both types of transformations.}
\label{tab:snr_prediction_real_075}
\begin{tabular}{lcccc}
\toprule
\rowcolor{gray!25}
\textbf{Parameter} & \textbf{Formula} & $\boldsymbol{\rho}$ & \textbf{Real} \\
\midrule
$\rho$               & SNR from Eq.~\eqref{eq:noise}                & 0.75   & ---   \\
$\rho^{-\nu}$        & ---                            & 1.044  & ---   \\
$B_{\text{eff}}$     & $B_s \cdot \rho^{-\nu}$        & 29.95  & ---   \\
$\kappa(1-\rho)^\mu$ & ---                            & 0.588  & ---   \\
$\hat{E}_t$          & $E_s + \kappa(1-\rho)^\mu$     & 2.928  & 2.732 \\
$\hat{\beta}_t$      & refitted                       & 0.277  & 0.252 \\
\midrule
\rowcolor{gray!25}
\multicolumn{4}{c}{\textbf{Fixed parameters (from source domain fitting)}} \\
\multicolumn{4}{c}{$E_s = 2.34$ (irreducible loss)} \\
\multicolumn{4}{c}{$\beta_s = 0.31$ (source data-scaling exponent)} \\
\multicolumn{4}{c}{$\nu = 0.13$, $\kappa = 2.35$, $\mu = 1.0$ (transformation-type parameters)} \\
\multicolumn{4}{c}{$\alpha_s = 0.27$, $A_s = 21.25$, $B_s = 28.69$} \\
\bottomrule
\end{tabular}
\end{table}

\begin{table}[h]
\centering
\caption{Cross-corpus scaling prediction (setting ii) results for the next-token prediction task, supplementary to Table~\ref{tab:nsl_pred}. We additionally compare different $\rho$ estimators against the ground truth.}
\label{tab:rho_estimators}
\begin{tabular}{lccccc}
\toprule
\rowcolor{gray!25}
\textbf{Parameter} & \textbf{Formula} & $\boldsymbol{\rho}$ \textbf{vocab} & $\boldsymbol{\rho}$ \textbf{3-gram} & $\boldsymbol{\rho}$ \textbf{compress} & \textbf{Ground Truth} \\
\midrule
$\rho$              & estimated                       & 0.71  & 0.57  & 0.54  & ---   \\
$\rho^{-\nu}$       & ---                             & 1.066 & 1.112 & 1.124 & ---   \\
$B_{\text{eff}}$    & $B_s \cdot \rho^{-\nu}$         & 47.97 & 50.04 & 50.58 & ---   \\
$\kappa(1-\rho)^\mu$ & ---                            & 0.757 & 1.122 & 1.201 & ---   \\
$\hat{E}_t$         & $E_s + \kappa(1-\rho)^\mu$      & 3.557 & 3.922 & 4.001 & 4.413 \\
$\Delta E$          & $|\hat{E} - E_{\text{real}}|$   & 0.856 & 0.491 & 0.412 & ---   \\
$\hat{\beta}_t$     & refitted                        & 0.298 & 0.319 & 0.321 & 0.325 \\
\midrule
\rowcolor{gray!25}
\multicolumn{6}{c}{\textbf{Fixed parameters (from source domain fitting)}} \\
\multicolumn{6}{c}{$A_s = 24.96$, $\alpha_s = 0.35$ (model-capacity terms, invariant)} \\
\multicolumn{6}{c}{$B_s = 45.02$, $\beta_s = 0.33$ (source data-scaling)} \\
\multicolumn{6}{c}{$E_s = 2.80$ (source irreducible loss)} \\
\multicolumn{6}{c}{$\nu = 0.19$, $\kappa = 2.61$, $\mu = 1.0$ (transformation-type parameters)} \\
\bottomrule
\end{tabular}
\end{table}

\begin{table}[h]
\centering
\caption{Within dataset noise injection ($\rho(T) = 0.75$) scaling prediction (setting i) results for the in-hospital mortality task, supplementary to Table~\ref{tab:nsl_pred}.}
\label{tab:snr_prediction}
\begin{tabular}{lcccc}
\toprule
\rowcolor{gray!25}
\textbf{Parameter} & \textbf{Formula} & \textbf{Values} & \textbf{Ground Truth} \\
\midrule
$\rho$               & SNR from Eq.~16                & 0.75  & ---   \\
$\rho^{-\nu}$        & ---                            & 1.035 & ---   \\
$B_{\text{eff}}$     & $B_s \cdot \rho^{-\nu}$        & 12.42 & ---   \\
$\kappa(1-\rho)^\mu$ & ---                            & 1.05  & ---   \\
$\hat{E}_t$          & $E_s + \kappa(1-\rho)^\mu$     & 2.54  & 2.36  \\
$\Delta E$           & $|\hat{E} - E_{\text{real}}|$  & 0.18  & ---   \\
$\hat{\beta}_t$      & refitted                       & 0.296 & 0.284 \\
\midrule
\rowcolor{gray!25}
\multicolumn{4}{c}{\textbf{Fixed parameters (from source domain fitting)}} \\
\multicolumn{4}{c}{$E_s = 1.491$ (irreducible loss)} \\
\multicolumn{4}{c}{$\beta_s = 0.322$ (source data-scaling exponent)} \\
\multicolumn{4}{c}{$\nu = 0.121$, $\kappa = 3.102$, $\mu = 0.730$ (transformation-type parameters)} \\
\bottomrule
\end{tabular}
\end{table}

\begin{table}[h]
\centering
\caption{Within dataset noise injection ($\rho(T) = 0.5$) scaling prediction (setting i) results for the in-hospital mortality task, supplementary to Table~\ref{tab:nsl_pred}.}
\label{tab:snr_prediction_05}
\begin{tabular}{lcccc}
\toprule
\rowcolor{gray!25}
\textbf{Parameter} & \textbf{Formula} & \textbf{Values} & \textbf{Ground Truth} \\
\midrule
$\rho$               & SNR from Eq.~16                & 0.5    & ---   \\
$\rho^{-\nu}$        & ---                            & 1.087  & ---   \\
$B_{\text{eff}}$     & $B_s \cdot \rho^{-\nu}$        & 13.04  & ---   \\
$\kappa(1-\rho)^\mu$ & ---                            & 1.878  & ---   \\
$\hat{E}_t$          & $E_s + \kappa(1-\rho)^\mu$     & 3.427  & 3.378 \\
$\Delta E$           & $|\hat{E} - E_{\text{real}}|$  & 0.049  & ---   \\
$\hat{\beta}_t$      & refitted                       & 0.251  & 0.244 \\
\midrule
\rowcolor{gray!25}
\multicolumn{4}{c}{\textbf{Fixed parameters (from source domain fitting)}} \\
\multicolumn{4}{c}{$E_s = 1.493$ (irreducible loss)} \\
\multicolumn{4}{c}{$\beta_s = 0.320$ (source data-scaling exponent)} \\
\multicolumn{4}{c}{$\nu = 0.124$, $\kappa = 3.101$, $\mu = 0.730$ (transformation-type parameters)} \\
\bottomrule
\end{tabular}
\end{table}

\begin{table}[h]
\centering
\caption{Within dataset noise injection ($\rho(T) = 0.25$) scaling prediction (setting i) results for the in-hospital mortality task, supplementary to Table~\ref{tab:nsl_pred}.}
\label{tab:snr_prediction_025}
\begin{tabular}{lcccc}
\toprule
\rowcolor{gray!25}
\textbf{Parameter} & \textbf{Formula} & \textbf{Values} & \textbf{Ground Truth} \\
\midrule
$\rho$               & SNR from Eq.~16                & 0.25   & ---   \\
$\rho^{-\nu}$        & ---                            & 1.178  & ---   \\
$B_{\text{eff}}$     & $B_s \cdot \rho^{-\nu}$        & 14.14  & ---   \\
$\kappa(1-\rho)^\mu$ & ---                            & 2.479  & ---   \\
$\hat{E}_t$          & $E_s + \kappa(1-\rho)^\mu$     & 3.933  & 4.146 \\
$\Delta E$           & $|\hat{E} - E_{\text{real}}|$  & 0.213  & ---   \\
$\hat{\beta}_t$      & refitted                       & 0.182  & 0.193 \\
\midrule
\rowcolor{gray!25}
\multicolumn{4}{c}{\textbf{Fixed parameters (from source domain fitting)}} \\
\multicolumn{4}{c}{$E_s = 1.487$ (irreducible loss)} \\
\multicolumn{4}{c}{$\beta_s = 0.323$ (source data-scaling exponent)} \\
\multicolumn{4}{c}{$\nu = 0.119$, $\kappa = 3.104$, $\mu = 0.730$ (transformation-type parameters)} \\
\bottomrule
\end{tabular}
\end{table}

\section{Experiment Details}\label{app:exp_details}
All experiments are run on $8\times$ NVIDIA A5500 GPUs with $24$~GB memory per GPU.

\subsection{Dataset and Task Information.} \label{subapp:data_details}
We first evaluate scaling behavior on standard language modeling tasks, including next-token prediction and neural machine translation. For next-token prediction, we use two widely adopted benchmarks: WikiText-103 \citep{merity2016pointer} and OpenWebText \citep{Gokaslan2019OpenWeb}. WikiText-103 comprises over 100 million tokens drawn from verified “Good” and “Featured” Wikipedia articles. OpenWebText is substantially larger, containing over 10 billion tokens across more diverse domains, and consists of high-quality web content curated from a large collection of URLs. For the neural machine translation task, we use subsets of the Colossal Clean Crawled Corpus (C4) \citep{dodge2021documenting}, which contains over 156 billion tokens collected from more than 365 million domains and has been used to train models such as T5 and the Switch Transformer \citep{fedus2022switch}. We evaluate scaling behavior for vision models using image classification tasks on standard benchmarks such as ImageNet-1K \citep{deng2009imagenet} and CIFAR-100 \citep{krizhevsky2009learning}. In addition, we evaluate scaling behavior on a speech recognition task using the LibriSpeech dataset \citep{panayotov2015librispeech}, a large-scale corpus derived from audiobooks that contains approximately 1,000 hours of English speech. We also study a healthcare-domain time-series classification task using the MIMIC-IV dataset \citep{johnson2023mimic}, a comprehensive database of de-identified electronic health records from patients admitted to medical centers. Our focus is on clinical time-series data, including high-frequency vital signs (e.g., heart rate and blood pressure) and laboratory measurements recorded during ICU stays.

\subsection{Implementation Details.}
For language modeling, we conduct pretraining experiments using a Llama 3 style causal Transformer architecture \citep{dubey2024llama}. We employ the GPT-2 tokenizer with a context length of 1024. For each scaling law study, we consider 6 dataset sizes and 5 model sizes for each corresponding dataset. We focus on two types of bijective transformations: linear transformations, which apply a random matrix to token IDs, and nonlinear transformations, which apply a multi-layer stochastic neural network to the resulting token embeddings before they are fed into the model; note that, the transformation network is not trained jointly with the model. For non-bijective transformations, we apply quantization and low-rank projection at varying levels to the token embeddings. For ImageNet-1K, we use Vision Transformer \citep{dosovitskiy2020image} models with four different capacities and six dataset sizes. For CIFAR-100, we instead use ResNet architectures \citep{he2016deep} with four different capacities and six dataset sizes. For the LibriSpeech dataset, we employ Wav2Vec2 \citep{baevski2020wav2vec} models with four different model sizes and five dataset sizes. 

\subsection{Hyper-parameters and Model Configurations}
Tables~\ref{tab:lm_datasets} and \ref{tab:vision_speech_datasets} summarize the dataset configurations across different modalities used in the scaling law experiments. Table~\ref{tab:model_configs} provides comprehensive details of the model architectures and training configurations employed across modalities.

\begin{table}[h]
\centering
\caption{Language modeling dataset configurations.}
\label{tab:lm_datasets}
\begin{tabular}{l p{9cm}}
\toprule
\textbf{Property} & \textbf{Value} \\
\midrule
\rowcolor{gray!25}
\multicolumn{2}{l}{\textbf{WikiText-103}} \\
Source              & HuggingFace \texttt{wikitext-103-v1} \\
Tokenizer           & GPT-2 \\
Vocabulary size     & 50,257 \\
Sequence length     & 1,024 tokens (chunked, non-overlapping) \\
Train sizes tested  & 100K, 300K, 1M, 3M, 10M, 30M, 100M, 300M, 1B, 3B tokens \\
Train batch size    & 32 \\
\midrule
\rowcolor{gray!25}
\multicolumn{2}{l}{\textbf{OpenWebText}} \\
Source              & HuggingFace \texttt{openwebtext} \\
Split               & 98\% train / 1\% val / 1\% test \\
Tokenizer           & GPT-2 \\
Vocabulary size     & 50,257 \\
Sequence length     & 1,024 tokens \\
Train sizes tested  & 3M, 10M, 30M, 100M, 300M, 1B, 3B, 10B tokens \\
\bottomrule
\end{tabular}
\end{table}

\begin{table}[h]
\centering
\caption{Vision and speech dataset configurations.}
\label{tab:vision_speech_datasets}
\begin{tabular}{l p{9cm}}
\toprule
\textbf{Property} & \textbf{Value} \\
\midrule
\rowcolor{gray!25}
\multicolumn{2}{l}{\textbf{ImageNet}} \\
Classes             & 1,000 \\
Image size          & $224 \times 224$ \\
Patch size (ViT)    & $16 \times 16 \rightarrow 196$ patches $+ 1$ CLS $= 197$ tokens \\
Normalization       & mean $= [0.485, 0.456, 0.406]$, std $= [0.229, 0.224, 0.225]$ \\
Train transform     & RandomResizedCrop(224) + RandomHorizontalFlip + Normalize \\
Val/test transform  & Resize(256) + CenterCrop(224) + Normalize \\
Train sizes tested  & 1K, 3K, 10K, 30K, 100K, 300K samples \\
Batch size          & 128 (train and eval) \\
Epochs              & 100 \\
Sampling            & Stratified, nested (smaller sets are true subsets of larger) \\
\midrule
\rowcolor{gray!25}
\multicolumn{2}{l}{\textbf{LibriSpeech}} \\
Train splits        & train-clean-100 + train-clean-360 + train-other-500 \\
Dev / test splits   & dev-clean / test-clean \\
Sample rate         & 16,000 Hz \\
Max audio length    & 20.0 seconds (truncated) \\
Avg utterance duration & 6.5 seconds \\
CTC vocabulary      & 32 (character-level) \\
CNN downsampling    & $320\times$ \\
Train sizes tested  & 10h $\approx$ 5,538 samples; 50h $\approx$ 27,692; 100h $\approx$ 55,384; 500h $\approx$ 276,923 \\
Batch size          & 32 \\
Metric              & WER + CTC loss \\
\bottomrule
\end{tabular}
\end{table}

\begin{table}[h]
\centering
\caption{Model architecture and training configurations across modalities.}
\label{tab:model_configs}
\small
\begin{tabular}{l p{3.6cm} p{3.6cm} p{3.6cm}}
\toprule
\textbf{Parameter / Property} & \textbf{Text (LM)} & \textbf{Image (ViT)} & \textbf{Speech (Wav2Vec2)} \\
\midrule
\rowcolor{gray!25}
\multicolumn{4}{l}{\textbf{Model Architecture}} \\
Size Index 0 & Hidden 16, Layers 2, Heads 2, FFN 64, Params $\sim$827K & ViT-Tiny: Embed 192, Depth 12, Heads 3, FFN 768, Params $\sim$5.5M & wav2vec2-tiny: Hidden 256, Layers 4, Heads 4, FFN 1024, Params $\sim$4M \\
Size Index 1 & Hidden 32, Layers 4, Heads 4, FFN 128, Params $\sim$1.7M & ViT-Small: Embed 384, Depth 12, Heads 6, FFN 1536, Params $\sim$22M & wav2vec2-small: Hidden 384, Layers 6, Heads 6, FFN 1536, Params $\sim$13M \\
Size Index 2 & Hidden 64, Layers 6, Heads 8, FFN 256, Params $\sim$3.6M & ViT-Base: Embed 768, Depth 12, Heads 12, FFN 3072, Params $\sim$86M & wav2vec2-medium: Hidden 512, Layers 8, Heads 8, FFN 2048, Params $\sim$30M \\
Size Index 3 & Hidden 128, Layers 8, Heads 8, FFN 512, Params $\sim$8.1M & ViT-Large: Embed 1024, Depth 24, Heads 16, FFN 4096, Params $\sim$307M & wav2vec2-large: Hidden 640, Layers 10, Heads 10, FFN 2560, Params $\sim$55M \\
Size Index 4 & Hidden 256, Layers 10, Heads 8, FFN 1024, Params $\sim$21M & --- & wav2vec2-xlarge: Hidden 768, Layers 12, Heads 12, FFN 3072, Params $\sim$91M \\
Size Index 5 & Hidden 512, Layers 12, Heads 12, FFN 2048, Params $\sim$64M & --- & --- \\
Size Index 6 & Hidden 640, Layers 24, Heads 16, FFN 2560, Params $\sim$151M & --- & --- \\
Size Index 7 & Hidden 768, Layers 36, Heads 16, FFN 3072, Params $\sim$294M & --- & --- \\
Size Index 8 & Hidden 1024, Layers 48, Heads 32, FFN 4096, Params $\sim$657M & --- & --- \\
\midrule
\rowcolor{gray!25}
\multicolumn{4}{l}{\textbf{Training Hyperparameters}} \\
Optimizer           & AdamW & AdamW & AdamW \\
Learning rate       & $5\!\times\!10^{-4}$ & $1\!\times\!10^{-4}$ & $5\!\times\!10^{-4}$ \\
$\beta_1, \beta_2$  & 0.9, 0.95 & 0.9, 0.999 & 0.9, 0.95 \\
Weight decay        & 0.01 & 0.05 & 0.01 \\
Grad clip           & 1 & 1 & 1 \\
LR schedule         & Linear warmup + linear decay to 0 & Linear warmup + linear decay to 0 & Linear warmup + linear decay to 0 \\
Warmup steps        & 1,000 & 1,000 & 1,000 \\
Precision           & FP16 & FP16 & FP16 \\
Max steps / epochs  & 1000 epochs & 100 epochs & 1000 epochs \\
\bottomrule
\end{tabular}
\end{table}

\section{Broader Impact}\label{app:impact}

This work aims to make neural scaling laws transferable: fit once on a well-resourced source domain and applied to settings where running a full sweep is infeasible. The most immediate beneficiaries are domains operating under structural data and compute constraints: clinical AI, scientific modeling, low-resource languages, and applications in regulated industries where privacy or cost prevents the kind of large-scale sweeps that yielded existing scaling laws. By providing principled forecasts of how performance will scale before training, our framework can reduce wasted compute, lower the energy footprint of model development, and democratize access to scaling-informed model design beyond labs with frontier-scale resources. We see particular value in healthcare, where over- or under-provisioned models carry direct patient-safety consequences and where data access is strictly capped by privacy regulations. Potential risks are modest but worth noting: predicted scaling curves are point estimates with their own uncertainty, and over-reliance on them, particularly in safety-critical domains, could substitute extrapolation for empirical validation. We therefore view our framework as a planning tool that complements, rather than replaces, careful evaluation at deployment scale.

\section{Dataset Licenses and Terms of Use}\label{app:license}

All datasets used in this work are publicly available and were used in accordance with their respective licenses and terms. \textbf{WikiText-103}~\citep{merity2016pointer} is released under the Creative Commons Attribution-ShareAlike (CC BY-SA 3.0) license. \textbf{OpenWebText}~\citep{Gokaslan2019OpenWeb} is distributed under the Creative Commons CC0 1.0 Universal (public domain) license, with the underlying content sourced from publicly accessible Reddit-linked URLs. \textbf{C4}~\citep{dodge2021documenting} is released by AllenAI under the Open Data Commons Attribution License (ODC-BY), with the underlying Common Crawl data subject to Common Crawl's terms of use. \textbf{ImageNet-1K}~\citep{deng2009imagenet} is available for non-commercial research and educational use under the ImageNet terms of access; we use it solely for academic research consistent with these terms. \textbf{CIFAR-100}~\citep{krizhevsky2009learning} is released for research purposes by the Canadian Institute for Advanced Research and is freely available for academic use. \textbf{LibriSpeech}~\citep{panayotov2015librispeech} is distributed under the Creative Commons Attribution 4.0 (CC BY 4.0) license, derived from public-domain LibriVox audiobooks. \textbf{MIMIC-IV}~\citep{johnson2023mimic} is released under the PhysioNet Credentialed Health Data License 1.5.0 and requires completion of CITI "Data or Specimens Only Research" training and a signed data use agreement; all authors who accessed MIMIC-IV completed the required credentialing, and no protected health information is reproduced in this paper. We use each dataset solely for the non-commercial academic research purposes permitted by its license.

\FloatBarrier
\end{document}